\documentclass[twocolumn]{svjour3}          
\smartqed  

\usepackage{amsmath,graphicx}
\usepackage{epsfig}
\usepackage{amssymb}
\usepackage{subfigure}
\usepackage{rotating}
\usepackage{epstopdf}
\usepackage{color}
\usepackage{algpseudocode}
\usepackage{algorithm}
\usepackage{stmaryrd}
\usepackage{url}
\usepackage{import}
\usepackage{booktabs}
\usepackage{multirow}

\title{Poisson noise reduction with non-local PCA}

\author{Joseph Salmon \and Zachary Harmany \and Charles-Alban Deledalle \and Rebecca Willett  }

\institute{ Joseph Salmon \at
Department LTCI, CNRS UMR 5141, Telecom Paristech\\ 
Paris, France\\
\email{joseph.salmon@telecom-paristech.fr}
\and
Zachary Harmany \at
Department of Electrical and Computer Engineering\\
University of Wisconsin-Madison\\
Madison, Wisconsin, USA\\
\email{harmany@wisc.edu}
\and
Rebecca Willett \at
  Department of Electrical and Computer Engineering\\
  Duke University\\
  Durham, NC, USA.\\
  \email{willett@duke.edu}
  \and
  Charles-Alban Deledalle \at
  IMB, CNRS-Universit\'{e} Bordeaux 1\\
  Talence, France\\
  \email{charles-alban.deledalle@math.u-bordeaux1.fr}  
}

\date{Received: date / Accepted: date}

\usepackage{amsmath}
\usepackage{dsfont}
\usepackage{amsmath}

\newcommand{\currentcaption}{}
\newcommand{\currentname}{}

\definecolor{darkgreen}{rgb}{.0,.638,.035}
\definecolor{darkred}{rgb}{.638,.0,.035}
\definecolor{purple}{rgb}{0.4,.1,.9}
\definecolor{badgerred}{RGB}{183,01,01}

\newcommand{\ie}{{\em i.e.,~}}
\newcommand{\eg}{{\em e.g.,~}}
\newcommand{\lcf}{{\em cf.~}}

\newcommand{\argmin}{\mathop{\mathrm{arg\,min}}}

\DeclareMathOperator{\PSNR}{PSNR}

\DeclareMathOperator{\MAE}{MAE}

\newcommand{\wh}[1]{\widehat{#1}}

\makeatletter

\newcommand*{\wpatch}{%
 \protect \@ifnextchar\bgroup\w@patch\w@@patch
}

\newcommand*{\w@patch}[1]{
  \ensuremath{\boldsymbol S \left( #1 \right)}%
}

\newcommand*{\w@@patch}{
  \ensuremath{\boldsymbol S}%
}

\newcommand*{\wpatchoppose}{%
 \protect \@ifnextchar\bgroup\my@patch\my@@patch
}

\newcommand*{\my@patch}[1]{
  \ensuremath{\check{(\boldsymbol S}\left( #1 \right)}%
}

\newcommand*{\my@@patch}{
  \ensuremath{\check{\boldsymbol S}}%
}

\makeatother

\DeclareMathOperator{\pen}{Pen}

\def\R{\mathbb{R}}

\def\N{\mathbb{N}}

\usepackage{dsfont}
\newcommand{\1}{\mathds{1}} 
\def\P{\mathbb{P}}
\def\E{\mathbb{E}}

\DeclareMathOperator{\Poisson}{Poisson}

\DeclareMathOperator{\Tr}{Tr}

\DeclareMathOperator{\diag}{diag}

\DeclareMathOperator{\Vect}{Vect}

\DeclareMathOperator{\sign}{sign}

\newcommand{\etaST}{{\mathbf{\eta}}_\text{\tiny ST}}

\usepackage{color}

\DeclareMathOperator{\iter}{iter}
\DeclareMathOperator{\stoping}{stop}
\DeclareMathOperator{\cond}{cond}
\DeclareMathOperator{\test}{test}
\DeclareMathOperator{\randn}{randn}

\begin{document}
\sloppy 

\maketitle

\begin{abstract}
	Photon-limited imaging arises when the number of photons collected by
	a sensor array is small relative to the number of detector elements.
	Photon limitations are an important 
	concern for many applications such as spectral imaging, night vision,
	nuclear medicine, and astronomy. 
  Typically a Poisson
  distribution is used to model these observations, and the inherent
  heteroscedasticity of the data combined with standard noise removal
  methods yields significant artifacts. This paper introduces a novel
  denoising algorithm for photon-limited images which combines
  elements of dictionary learning and sparse patch-based representations of images. 
  The method employs both an adaptation of Principal
  Component Analysis (PCA) for Poisson noise and recently developed
  sparsity-regularized convex optimization algorithms for photon-limited images.
  A comprehensive empirical evaluation of the proposed
  method helps characterize the performance of this approach relative
  to other state-of-the-art denoising methods.  The results reveal
  that, despite its conceptual simplicity, Poisson PCA-based denoising appears to be
  highly competitive in very low light regimes.

\keywords{Image denoising \and PCA \and Gradient methods \and Newton's method \and Signal representations}

\end{abstract}

%
\section{Introduction, model, and notation}
\label{sec:intro}

In a broad range of imaging applications, observations correspond to
counts of photons hitting a detector array, and these counts can be
very small. For instance, in night vision, infrared, and certain
astronomical imaging systems, there is a limited amount of available
light. Photon limitations can even arise in well-lit environments when
using a spectral imager which characterizes the wavelength of each
received photon. The spectral imager produces a three-dimensional data
cube, where each voxel in this cube represents the light intensity at
a corresponding spatial location and wavelength. As the spectral
resolution of these systems increases, the number of available photons
for each spectral band decreases. Photon-limited imaging algorithms
are designed to estimate the underlying spatial or spatio-spectral
intensity underlying the observed photon counts.

There exists a rich literature on image estimation or denoising
methods, and a wide variety of effective tools.  The photon-limited
image estimation problem is particularly challenging because
the limited number of available photons introduces intensity-dependent
Poisson statistics which require specialized algorithms and analysis
for optimal performance.  Challenges associated with low photon count
data are often circumvented in hardware by designing systems which
aggregate photons into fixed bins across space and wavelength (\ie
creating low-resolution cameras). If the bins are large enough, the
resulting low spatial and spectral resolution cannot be overcome. High-resolution
observations, in contrast, exhibit significant non-Gaussian noise
since each pixel is generally either one or  zero  (corresponding to whether 
or not a photon is counted by the detector), and
conventional algorithms which neglect the effects of photon noise will
fail. Simply transforming Poisson data to produce data with approximate Gaussian
noise (via, for instance, the variance stabilizing Anscombe transform
\cite{Anscombe48,Makitalo_Foi11} or Fisz transform
\cite{Fisz55,Fryzlewicz_Nason01}) can be effective when the number photon counts are uniformly high
\cite{Boulanger_Kervrann_Bouthemy_Elbau_Sibarita_Salamero10,Zhang_Fadili_Starck08}. 
However, when photon counts are very low these approaches may suffer, as
shown later in this paper. 

This paper demonstrates how advances in low-dimensional modeling and
sparse Poisson intensity reconstruction algorithms can lead to
significant gains in photon-limited (spectral) image accuracy at the
resolution limit. The proposed method combines Poisson Principal
Component Analysis (Poisson-PCA -- a special case of the
Exponential-PCA
\cite{Collins_Dasgupta_Schapire02,Singh_Gordon08}) 
and sparse Poisson intensity estimation methods
\cite{Harmany_Marcia_Willett12} in a non-local estimation framework.
We detail the
targeted optimization problem which incorporates the heteroscedastic
nature of the observations and present results improving upon
state-of-the-art methods when the noise level is particularly high. We coin our method Poisson Non-Local Principal
Component Analysis (Poisson NLPCA).

Since the introduction of non-local methods for image denoising
\cite{Buades_Coll_Morel05}, these methods have proved to outperform
previously considered approaches
\cite{Aharon_Elad_Bruckstein06,Dabov_Foi_Katkovnik_Egiazarian07,Mairal_Bach_Ponce_Sapiro_Zisserman09,Dabov_Foi_Katkovnik_Egiazarian09}
(extensive comparisons of recent denoising method can be found for Gaussian noise 
in \cite{Katkovnik_Foi_Egiazarian_Astola10,Lebrun_Colom_Buades_Morel12}).	
Our work is inspired by recent methods combining PCA with patch-based approaches
\cite{Muresan_Parks03,Zhang_Dong_Zhang_Shi10,Deledalle_Salmon_Dalalyan11}
for the Additive White Gaussian Noise (AWGN) model, with natural extensions to
spectral imaging \cite{Danielyan_Foi_Katkovnik_Egiazarian10}. A major
difference between these approaches and our method is that we directly
handle the Poisson structure of the noise, without any
``Gaussianization'' of the data. Since our method does not use
a quadratic data fidelity term, the singular value decomposition (SVD) cannot
be used to solve the minimization. Our direct approach is particularly
relevant when the image suffers from a high noise level (\ie low
photon emissions). 

\subsection{Organization of the paper} In Section
\ref{sec:formulation}, we describe the mathematical framework.  In
Section \ref{sec:expo}, we recall relevant basic properties of the
exponential family, and propose an optimization formulation for
matrix factorization. Section \ref{sec:newton} provides an algorithm
to iteratively compute the solution of our minimization problem.  In Section
\ref{sec:clustering}, an important clustering step is introduced both
to improve the performance and the computational complexity of our algorithm.
Algorithmic details and experiments are reported in Section
\ref{sec:algorithm} and \ref{sec:experiments}, and we conclude in
Section \ref{seq:conclusion}.


%


\subsection{Problem formulation}\label{sec:formulation}
For an integer $M>0$, the set $\{1,\ldots\!,M \}$ is denoted
$\llbracket 1,M \rrbracket$.  For $i \in \llbracket 1,M \rrbracket$,
let $y_i$ be the observed pixel values obtained through an image
acquisition device. We consider each $y_i$ to be an independent random
Poisson variable whose mean $f_i\geq0$ is the underlying intensity
value to be estimated. Explicitly, the discrete Poisson probability of
each $y_i$ is 
\begin{align}
\P(y_i|f_i)= 
	\displaystyle\frac{f_i^{y_i}}{y_i !}e^{-f_i},
\end{align}
where $0!$ is understood to be 1 and $0^0$ to be $1$.

A crucial property of natural images is their ability to be accurately
represented using a concatenation of patches, each of which is a
simple linear combination of a small number of representative
atoms. One interpretation of this property is that the patch
representation exploits self-similarity present in many images, as
described in AWGN settings
\cite{Dabov_Foi_Katkovnik_Egiazarian07,Mairal_Bach_Ponce_Sapiro_Zisserman09,Dabov_Foi_Katkovnik_Egiazarian09}.
Let $Y$ denote the $M\times N$ matrix of all the vectorized $\sqrt{N}
\times \sqrt{N}$ overlapping patches (neglecting border issues)
extracted from the noisy image, and let $F$ be defined similarly for
the true underlying intensity. Thus $Y_{i,j}$ is the $j$th pixel in the
$i$th patch.

Many methods have been proposed
to represent the collection of patches  
in a low dimensional space in the same spirit
as PCA.   
We use the framework considered in
\cite{Collins_Dasgupta_Schapire02,Singh_Gordon08}, that deals with data 
well-approximated by random variables drawn from exponential family
distributions. In particular, we use Poisson-PCA, which we briefly
introduce here before giving more details in the next section. 
With Poisson-PCA, one aims to approximate $F$ by:
\begin{equation}
	F_{i,j} \approx \exp ( [UV] _{i,j} )
  \quad \forall  (i,j) \in \llbracket 1,M \rrbracket  \times 
  \llbracket 1,N \rrbracket,
  \vspace{-0.2cm}
\label{eq:expApprox}
\end{equation}where
\begin{itemize}
\item $U$ is the $M \times \ell$ matrix of coefficients;
\item $V$ is the $\ell \times N$ matrix representing the dictionary
  components or axis. The rows of $V$ represents the dictionary
  elements; and
  \item  $\exp(UV)$ is the element-wise exponentiation of $UV$:
    $\exp\big([UV]_{i,j}\big) := \big[\exp(UV)\big]_{i,j}$.
\end{itemize}
The approximation in \eqref{eq:expApprox} is different than the
approximation model used in similar methods based on AWGN, where
typically one assumes $F_{i,j} \approx [UV]_{i,j}$ (that is, without
exponentiation). Our exponential model allows us to circumvent
challenging issues related to the nonnegativity of $F$ and thus
facilitates significantly faster algorithms.

The goal is to compute an estimate of the form \eqref{eq:expApprox}
from the noisy patches $Y$.  We assume that this approximation is
accurate for $\ell \ll M$, whereby restricting the rank $\ell$ acts to regularize the solution.  In the following section we elaborate on this low-dimensional representation.



\section{Exponential family and matrix factorization}\label{sec:expo}

We present here the general case of matrix factorization for an
exponential family, though in practice we only use this framework for
the Poisson and Gaussian cases.  We describe the idea for a general exponential family because our proposed 
method considers Poisson noise, but we also develop an analogous  method (for comparison purposes) based on 
an Anscombe transform of the data and a Gaussian noise model.
 The solution we focus on
follows the one introduced by \cite{Collins_Dasgupta_Schapire02}.
Some more specific details can be found in
\cite{Singh_Gordon08,Singh_Gordon08b} about matrix factorization for
exponential families.

\subsection{Background on the exponential family}
We assume that the observation space $\mathcal{Y}$ is equipped with a
$\sigma$-algebra $\mathcal{B}$ and a dominating $\sigma$-finite measure $\nu$ on $(\mathcal{Y}, \mathcal{B})$.
Given a positive integer $n$, let $\phi$: $\mathcal{Y} \rightarrow \R^n $ be a measurable
function, and let $\phi_k$, $k = 1, 2,\cdots,n$ denote its components:
$\phi(y)= \big( \phi_1(y), \cdots,\phi_n(y) \big)$.

Let $\Theta$ be defined as the set of all $\theta \in \R^n$ such that
$\int_{\mathcal{Y} }\exp( \langle \theta  |  \phi (y) \rangle ) d\nu < \infty $. We assume
it is convex and open in this paper.
We then have the following definition:

\begin{definition}
  An {\em exponential family with sufficient statistic $\phi$} is the
  set $\mathcal{P}(\phi)$ of probability distributions w.r.t.\ the measure $\nu$ on
  $(\mathcal{Y}, \mathcal{B})$  parametrized by $\theta \in \Theta$,
  such that
 each probability density function $p_\theta \in
  \mathcal{P}(\phi)$  can be expressed as
\begin{equation}
p_\theta(y) =\exp \left\{ \langle \theta  |  \phi (y) \rangle -\Phi (\theta)\right\},
\end{equation}
where 
\begin{equation}
\Phi (\theta) =\log \int_{\mathcal{Y} } \exp \left\{ \langle \theta  |  \phi (y) \rangle \right\} d\nu(y).
\end{equation}
The parameter $\theta \in \Theta$ is called the {\em natural parameter} of
$\mathcal{P}(\phi)$, and the set $\Theta$ is called the {\em natural
parameter space}. The function $\Phi$ is called the {\em log partition
function}.  We denote by $\E_\theta[\cdot]$ the expectation
w.r.t.\ $p_\theta$:
\begin{align*}
\E_\theta[g(X)]= \int_{\mathcal{X}} g(y)\left( \exp( \langle \theta  |  \phi (y) \rangle ) -\Phi(\theta) \right) d\nu (y).
\end{align*}
\end{definition}

\begin{example}
Assume the data are independent (not necessarily identically distributed) 
Gaussian random variables with means $\mu_i$ and  (known) variances $\sigma^2$.
Then the parameters are:
$\forall y \in \R^n, \phi(y)=y$,
$\Phi(\theta)=\sum_{i=1}^n \theta_i^2/2 \sigma^2$ and $\nabla\Phi(\theta)=(\theta_1/\sigma^2,\cdots,\theta_n/\sigma^2)$ and $\nu$ is 
the Lebesgue measure on $\R^n$
(\lcf \cite{Nielsen_Garcia09} for more 
details on the Gaussian distribution, possibly with non-diagonal covariance matrix).


\end{example}

\begin{example}
For Poisson distributed data (not necessarily identically distributed), 
the parameters are the following:
$\forall y \in \R^n, \phi(y)=y$, and  
$\Phi(\theta)=\langle  \exp(\theta) | \1_n \rangle=\sum_{i=1}^n e^{\theta_i}$,  where
$\exp$ is the component-wise  exponential function:
\begin{equation}\label{eq:EXP}
 \exp:(\theta_1,\cdots,\theta_n) \mapsto(e^{\theta_1},\cdots,e^{\theta_n}),
\end{equation} 
and 
$\1_n$ is the vector $(1,\cdots,1)^\top \in \R^n$. Moreover
$
\nabla\Phi(\theta)=\exp(\theta)
$
and $\nu$ is the counting measure on $\N$ weighted by $e/n!$.
\end{example}

\begin{remark}
The standard parametrization is usually different for Poisson distributed data,
and this family is often parametrized by  the rate parameter $f=\exp (\theta)$.
\end{remark}

\subsection{Bregman divergence }
The general measure of proximity we use in our analysis relies on Bregman
divergence \cite{Bregman67}. For exponential families, 
the relative entropy (Kullback-Leibler divergence) between 
$p_{\theta_1}$ and $p_{\theta_2}$ in $\mathcal{P}(\phi)$, defined as 
\begin{equation} 
D_\Phi(p_{\theta_1} ||  p_{\theta_2}) =
\int_{\mathcal{X}} p_{\theta_1} \log(p_{\theta_1} /p_{\theta_2} )d\nu, 
\end{equation}
can be simply written as a function of the natural parameters:
\begin{align*}
D_\Phi(p_{\theta_1} ||  p_{\theta_2}) = \Phi(\theta_2 ) - \Phi(\theta_1) -
\langle \nabla \Phi(\theta_1 ) | \theta_2 - \theta_1 \rangle .
\end{align*}
From the last equation, we have that 
the mapping $D_\Phi : \Theta \times \Theta \rightarrow \R$, defined
by $D_\Phi (\theta_1, \theta_2 ) = D_\Phi(p_{\theta_2}|| p_{\theta_2} )$, is a Bregman divergence.

\begin{example}\label{example:Gaussian}
For Gaussian distributed observations with unit variance and zero mean, the Bregman divergence can be written: 
\begin{equation}
D_G (\theta_1, \theta_2 )=\|\theta_1 - \theta_2 \|_2^2.
\label{eq:bregman_gaussian}
\end{equation}
\end{example}

\begin{example}\label{example:poisson}
For Poisson distributed observations, the Bregman divergence can be written:
\begin{equation}
D_P (\theta_1, \theta_2 )=\langle  
\exp(\theta_2)-\exp(\theta_1) | \1_n \rangle - \langle  \exp (\theta_1)| \theta_2-\theta_1 \rangle .
\label{eq:bregman_poisson}
\end{equation} 
\end{example}

We  define the matrix Bregman divergence as
\begin{multline}
D_\Phi(X|| Y)=\Phi(Y) - \Phi(X) \\-
\Tr \Big(  (\nabla \Phi(X))^\top (X -Y)  \Big), \label{eq:breg}
\end{multline}
for any (non necessarily square)  matrices $X$ and $Y$ of size $M \times N$.

\subsection{Matrix factorization and dictionary learning}

Suppose that one observes $Y\in \R^{M\times N}$, and let $Y_{i,:}$
denote the $i$th patch in row-vector form.  We would like to
approximate the underlying intensity $F$ by a combination of some
vectors, atoms, or dictionary elements $V=[v_1,\cdots,v_\ell]$, where each
patch uses different weights on the dictionary elements. In other
words, the $i$th patch of the true intensity, denoted $F_{i,:}$, is
approximated as $\exp(u_i V)$, where $u_i$ is the $i$th row of $U$ and
contains the dictionary weights for the $i$th patch.  Note that we perform this factorization in the natural parameter space, which is why we use the exponential function in the
formulation given in Eq. \eqref{eq:expApprox}.


Using the divergence defined in \eqref{eq:breg}
our objective is to find $U$ and $V$ minimizing the following criterion:
\begin{align*}
D_\Phi(Y||UV)
=& \sum_{j=1}^M \Phi(u_j V)-  Y_{j,:}  -
\langle Y_{j,:} |   u_j V -Y_{j,:}  \rangle \, .
\end{align*}

%
%
In the Poisson case, the framework introduced in \cite{Collins_Dasgupta_Schapire02,Singh_Gordon08} uses the Bregman divergence in Example \ref{example:poisson} and amounts to minimizing 
the following loss function 
\begin{equation}\label{eq:original_poisson_KL}
L(U,V)= \sum_{i=1}^M \sum_{j=1}^N  \exp(UV)_{i,j} -Y_{i,j} (UV)_{i,j}
\end{equation}
with respect to the matrices  $U$ and $V$. 
Defining the corresponding minimizers of the biconvex problem
\begin{equation}\label{eq:mini_program_poisson}
(U^*,V^*) \in \argmin_{(U,V) \in \R^{M\times \ell} \times \R^{\ell\times N }}L(U,V) \,,\vspace{-0.2cm}
\end{equation}
our image intensity estimate is
\begin{equation}
 \wh{F}=\exp (U^* V^*) \, . 
\end{equation}
This is what we call Poisson-PCA (of order $\ell$) in the remainder of
the paper.

\begin{remark}
  The classical PCA (of order $\ell$) is obtained using the Gaussian
  distribution, which leads to solving the same minimization as in
  Eq.~\eqref{eq:mini_program_poisson}, except that $L$ is replaced by
 \begin{equation*}
\tilde{L}(U,V)= \sum_{i=1}^M \sum_{j=1}^N  \left( (UV)_{i,j} -Y_{i,j} \right)^2 \, . 
\end{equation*}
\end{remark}


\begin{remark}
  The problem as stated is non-identifiable, as scaling the dictionary
  elements and applying an inverse scaling to the coefficients would
  result in an equivalent intensity estimate.  Thus, one should
  normalize the dictionary elements so that the coefficients cannot be
  too large and create numerical instabilities.  The easiest solution
  is to impose that the atoms $v_i$ are normalized w.r.t.\ the
  standard Euclidean norm, \ie for all $i\in \{1,\cdots,\ell\}$ one ensures
  that the constraint $\|v_i\|_2^2=\sum_{j=1}^n V_{i,j}^2=1$ is
  satisfied.  In practice though, relaxing this constraint modifies
  the final output in a negligible way while helping to keep the computational complexity low.
\end{remark}

\section{Newton's method for minimizing $L$} \label{sec:newton}

Here we follow the approach proposed by
\cite{Gordon03,Roy_Gordon_Thrun05} that consists in using Newton steps
to minimize the function $L$. Though $L$ is not jointly
convex in $U$ and $V$, when fixing one variable and keeping the other
fixed the partial optimization problem is convex (\ie the problem is
biconvex). Therefore we consider Newton updates on the partial
problems.  To apply Newton's method, one needs to invert the Hessian matrices
with respect to both  $U$ and $V$, defined by
$H_U=\nabla_{U}^2 L(U,V)$ and $H_V=\nabla_{V}^2 L(U,V)$. Simple
algebra leads to the following closed form expressions for the
components of these matrices (for notational simplicity 
we use pixel coordinates to index the entries of the Hessian):
\begin{equation*}
\frac{\partial^2 L(U,V)}{\partial U_{a,b} \partial U_{c,d}} =\left\{
    \begin{array}{ll}
        \displaystyle\sum_{j=1}^N  \exp(UV)_{a,j} V^2_{b,j}, & \mbox{if } (a,b)=(c,d), \\
        0 & \mbox{otherwise,}
    \end{array}
\right. 
\end{equation*}
and
\begin{equation*}
\frac{\partial^2 L(U,V)}{\partial V_{a,b} \partial V_{c,d}}=\left\{
    \begin{array}{ll}
        \displaystyle\sum_{i=1}^M  U^2_{i,a} \exp(UV)_{i,b}, &\mbox{if } (a,b)=(c,d), \\
        0&\mbox{otherwise,}
    \end{array}
\right. 
\end{equation*}
where both partial Hessians can be represented as diagonal matrices (\lcf Appendix~\ref{app:hess} for more details).
  
We propose to update the rows of $U$ and columns of $V$ as proposed in
\cite{Roy_Gordon_Thrun05}. We introduce the function $\Vect_C$ that
transforms a matrix into one single column (concatenates the columns),
and the function $ \Vect_R$ that transforms a matrix into a single row
(concatenates the rows). Precise definitions are given in 
Appendix~\ref{app:newton}.  The updating step for $U$ and $V$ are then given
respectively by
\begin{equation*}
\Vect_R(U_{t+1})=\Vect_R(U_t) - \Vect_R \big(\nabla_{U}L(U_t,V_t)\big)H_{U_t}^{-1} 
\end{equation*}
and
\begin{equation*}
 \Vect_C(V_{t+1})=\Vect_C(V_t) -H_{V_t}^{-1} \Vect_C \big(\nabla_{V}L(U_t,V_t)\big) \, .\label{eq:vectru}
\end{equation*}
Simple algebra (\lcf Appendix~\ref{app:newton} or  \cite{Gordon03} for more details) leads to the 
following updating rules for the $i$th row of $U_{t+1}$ (denoted $U_{t+1,i,:}$): 
\begin{equation}\label{eq:update_row_U_poisson}
U_{t+1,i,:}=U_{t,i,:} -(\exp(U_tV_t)_{i,:} -Y_{i,:})V_t^\top (V_t D_{i} V_t^\top)^{-1}~,
\end{equation}
where $D_{i}=\diag \big( \exp(U_t V_t)_{i,1},\ldots, \exp(U_tV_t)_{i,N}  \big)$
is a diagonal matrix of size  $N\times N$. The updating rule for $V_{t,:,j}$, the $j$th column  of $V_t$,  is computed 
in a similar way, leading to 
\begin{multline}\label{eq:update_row_V_poisson}
V_{t+1,:,j}=
V_{t,:,j} - \\ (U_{t+1} ^\top  E_{j} U_{t+1} )^{-1} U_{t+1}^\top (\exp(U_{t+1}V_t)_{:,j} -Y_{:,j}),
\end{multline}
where $E_{j}=\diag \big( \exp(U_{t+1} V_t)_{1,j},\ldots,
\exp(U_{t+1}V_t)_{M,j} \big)$ is a diagonal matrix of size $M\times
M$. More details about the implementation are given in Algorithm
\ref{alg:pseudo-code}.

\begin{algorithm}
\centering
\begin{minipage}{\linewidth}
\begin{small}
\begin{algorithmic}[0]
  \State \textbf{Inputs:} Noisy pixels $y_i$ for $i=1,\dots,M$
  \State \textbf{Parameters:} Patch size $\sqrt{N} \times \sqrt{N}$, number of clusters $K$, number of components $\ell$,
maximal number of iterations $N_{\iter}$
  \State \textbf{Output:} estimated image $\wh{f}$ 
 \State \textbf{Method:} 
 \State Patchization: create  the collection of patches for the  noisy image $Y$
  \State Clustering: create $K$ clusters  of patches using K-Means
\State The $k$th cluster (represented by a matrix $Y^k$) has $M_k$ elements
  \ForAll {cluster $k$ }
\State Initialize $U_0=\randn(M_k,\ell)$ and   $V_0=\randn(\ell,N)$
\While{$t \leq N_{\iter} $ and $\test > \varepsilon_{\stoping}$}
    \ForAll {$i \leq M_k $}
  \State Update  the $i$th row of $U$ using \eqref{eq:update_row_U_poisson}  or \eqref{eq:update_row_U_poisson_l1}-\eqref{eq:update_sparseU}
    \EndFor
    \ForAll {$j \leq \ell $}
  \State Update  the $j$th column of $V$ using \eqref{eq:update_row_V_poisson}
    \EndFor
  \State $t:=t+1$
   \EndWhile
\State $\wh{F}^k=\exp(U_t V_t)$  
\EndFor
\State Concatenation: fuse the collection of denoised patches $\wh{F}$
\State Reprojection: average the various pixel estimates due to overlaps to get an image estimate: $\wh{f}$
\end{algorithmic}
\caption{Poisson NLPCA/ NLSPCA\label{alg:pseudo-code}}
\end{small}
\end{minipage}
\end{algorithm}


\section{Improvements through $\ell_1$ penalization }
%
%

A possible alternative to minimizing
Eq. \eqref{eq:original_poisson_KL}, consists of minimizing a penalized
version of this loss, whereby a sparsity constraint is imposed on the
elements of $U$ (the dictionary coefficients).  Related ideas have
been proposed in the context of sparse PCA
\cite{Zou_Hastie_Tibshirani06}, dictionary learning
\cite{Lee_Battle_Raina_Ng07}, and matrix factorization
\cite{Mairal_Bach_Ponce_Sapiro_Zisserman09,Mairal_Bach_Ponce_Sapiro10} in the Gaussian case.  Specifically,
we minimize
\begin{align}\label{eq:poisson_constrained_KL}
L^{\pen}(U,V)=L(U,V)+\lambda \pen(U), 
\end{align}
where $\pen(U)$ is a penalty term that ensures we use only a few
 dictionary elements to represent each patch. The parameter $\lambda$ controls the
trade-off between data fitting and sparsity. We focus on the following
penalty function:
\begin{equation}
 \pen(U)=\displaystyle \sum_{i,j} |U_{i,j}|
\end{equation}
We refer to the method as the Poisson Non-Local Sparse PCA (NLSPCA).



The algorithm proposed in \cite{Mairal_Bach_Ponce_Sapiro10} can be
adapted with the SpaRSA step provided in
\cite{Wright_Nowak_Figueiredo09}, or in our setting by using its
adaptation to the Poisson case -- SPIRAL 
\cite{Harmany_Marcia_Willett12}. First one should note that the
updating rule for the dictionary element, \ie Equation
\eqref{eq:update_row_V_poisson}, is not modified. Only the coefficient
update, \ie Equation \eqref{eq:update_row_U_poisson} is modified as
follows:
\begin{equation}\label{eq:update_row_U_poisson_l1}
 U_{t+1,:}=\argmin_{u \in  \R^\ell}  \langle \exp(u V_{t}) | \1 \rangle -\langle u V_{t} | Y_{t+1,:} \rangle + \lambda \| u\|_1.
\end{equation}
For this step, we use the SPIRAL approach. This leads to the following updating 
rule for the coefficients:
\begin{equation}
\begin{aligned}
\label{eq:quadratic_approx}
U_{t+1,:}= & \argmin_{z \in  \R^\ell }  && \frac{1}{2} \| z- \gamma_t \|_2^2+ \frac{\lambda}{\alpha_t} \| z\|_1,& \\
&\text{subject to} &&\gamma_t =  U_{t,:} - \frac{1}{\alpha_t} \nabla_U f (U_{t,:}) .&
\end{aligned}
\end{equation}
where $\alpha_t>0$ and the function $f$ is defined by
\begin{equation*}
f (u)=\langle \exp(u V_{t}) | \1 \rangle -\langle u V_{t} | Y_{t+1,:} \rangle . 
\end{equation*}
The gradient can thus be expressed as
\begin{equation*}
\nabla f (u)= \big( \exp(u V_{t+1}) - Y_{t+1,:} \big) V_{t+1}^\top . 
\end{equation*}
Then the solution of the problem \eqref{eq:quadratic_approx}, is simply
\begin{equation}
U_{t+1,:}=\etaST \left(\gamma_t,\frac{\lambda}{\alpha_t}\right) \label{eq:update_sparseU}
\end{equation} 
where $\etaST$ is the soft-thresholding function 
 $\etaST(x,\tau) =\sign(x) \cdot (|x|-\tau)_+ $.

%
%
%

Other methods than SPIRAL for solving the Poisson $\ell_1$-constrained problem could be investigated, 
\eg  Alternating Direction Method of Multipliers
 (ADMM) algorithms for $\ell_1$-minimization  
(\lcf \cite{Yin_Osher_Goldfarb_Darbon08,Boyd_Parikh_Chu_Peleato_Eckstein11}, or one
specifically adapted to Poisson noise \cite{Figueiredo_BioucasDias10}), though choosing the augmented Lagrangian parameter for these methods can be challenging in practice.

\section{Clustering step}\label{sec:clustering}
Most strategies apply matrix factorization on patches extracted from the entire image.
A finer strategy consists in first performing a clustering step, and then applying
matrix factorization on each cluster.
Indeed, this avoids grouping dissimilar patches of the image, and allows us to represent
the data within each cluster with a lower dimensional dictionary.
This may also improve on the computation time of the dictionary.
In \cite{Dabov_Foi_Katkovnik_Egiazarian07,Mairal_Bach_Ponce_Sapiro_Zisserman09},
the clustering is based on a geometric partitioning of the image.
This improves on the global approach but may results in poor estimation where
the partition is too small. Moreover, this approach remains local and cannot exploit
the redundancy inside similar disconnected regions.
We suggest here using a non-local approach where
the clustering is directly performed in the patch domain similarly to \cite{Chatterjee_Milanfar11}.
Enforcing similarity inside non-local groups of patches
results in a more robust low rank representation of the data,
decreasing the size of the matrices to be factorized, and leading
to efficient algorithms.
Note that in \cite{Deledalle_Salmon_Dalalyan11}, the authors studied an hybrid approach
where the clustering is driven in a hierarchical image domain as well as in the patch domain to
provide both robustness and spatial adaptivity. We have not considered this
approach since, while increasing the computation load,
it yields to significant improvements particularly at low noise levels, which are not the main focus of this paper.

For clustering we have compared two solutions: one using only a simple $K$-means on the original data,
and one performing a Poisson $K$-means.
In similar fashion for adapting PCA for exponential families, the $K$-means clustering
algorithm can also be generalized using Bregman divergences; this is called Bregman clustering \cite{Banerjee_Merugu_Dhillon_Ghosh05}. This
approach, detailed in Algorithm \ref{alg:clustering}, has an EM (Expectation-Maximization) flavor and is
proved to converge in a finite number of steps.

The two variants we consider differ only in the choice of the divergence $d$ used to compare elements $x$ with respect to the centers of the clusters $x_C$:
\begin{itemize}
\item Gaussian: Uses the divergence defined in \eqref{eq:bregman_gaussian}:
\begin{equation*}
d(f,f_C)=D_G(f,f_C)=\|f-f_C\|_2^2.
\end{equation*}

\item Poisson: Uses the divergence defined in \eqref{eq:bregman_poisson}:
\begin{equation*}
d(f,f_C)=D_P(\log(f),\log(f_C))=\sum_j f^j_C -f^j \log(f_C^j)
\end{equation*} 
where the $\log$ is understood element-wise (note that the difference with \eqref{eq:bregman_poisson} is only due to a different parametrization here). 

\end{itemize}
In our experiments, we have used a small number (for instance $K=14$) of clusters fixed in advance.

\begin{algorithm}
\centering
\begin{minipage}{\linewidth}
\begin{small}
\begin{algorithmic}[0]
  \State \textbf{Inputs:} Data points: $(f_i)_{i=1}^M \in \R^N$, number of clusters: $K$, 
Bregman divergence: $d:\R^N\times\R^N \mapsto \R^+$
  \State \textbf{Output:} Clusters centers: $(\mu_k)_{k=1}^K$, partition associated : $(\mathcal{C}_k)_{k=1}^K$
  \State \textbf{Method:}
 \State Initialize $(\mu_k)_{k=1}^K$ by randomly selecting $K$ elements among $(f_i)_{i=1}^M$
\Repeat 
  \State{(\textit{The Assignment step: Cluster updates})}
  \State Set $\mathcal{C}_k:=\emptyset, 1\leq k \leq K$
  \For {$i=1,\cdots,M$}
  \State {$\mathcal{C}_{k^*}:=\mathcal{C}_{k^*} \cup \{f_i\}$}
  \State {where $k^*=\displaystyle\argmin_{k'=1,\cdots,K} d(f_i,\mu_{k'})$}
\EndFor
\State{(\textit{The Estimation step: Center updates})}
\For {$k=1,\cdots, K $}
  \State $\mu_k:=\frac{1}{\# \mathcal{C}_k } \displaystyle\sum_{f_i \in \mathcal{C}_k} f_i$ 
\EndFor
\Until{convergence}
\end{algorithmic}
\caption{Bregman hard clustering \label{alg:pseudo-code-Kmeans}\label{alg:clustering}}
\end{small}
\end{minipage}
\end{algorithm}

In the low-intensity setting we are targeting, clustering on the raw data 
may yield poor results.  
A preliminary
image estimate might be used for performing the clustering, especially if
one has a fast method giving a satisfying denoised image.  For instance,
one can apply the Bregman hard clustering on the denoised images obtained
after having performed the full Poisson NLPCA on the noisy data.  This
approach was the one considered in the short version of this paper
\cite{Salmon_Deledalle_Willett_Harmany12}, where we were using
only the classical $K$-means.  However, we have noticed
that using the Poisson $K$-means instead leads
to  a significant improvement. Thus, the benefit of iterating the clustering is lowered.
In this version, we do not consider such iterative refinement of the
clustering.  The entire algorithm is summarized in Fig.~\ref{fig:patch_pca_architecture}.


%

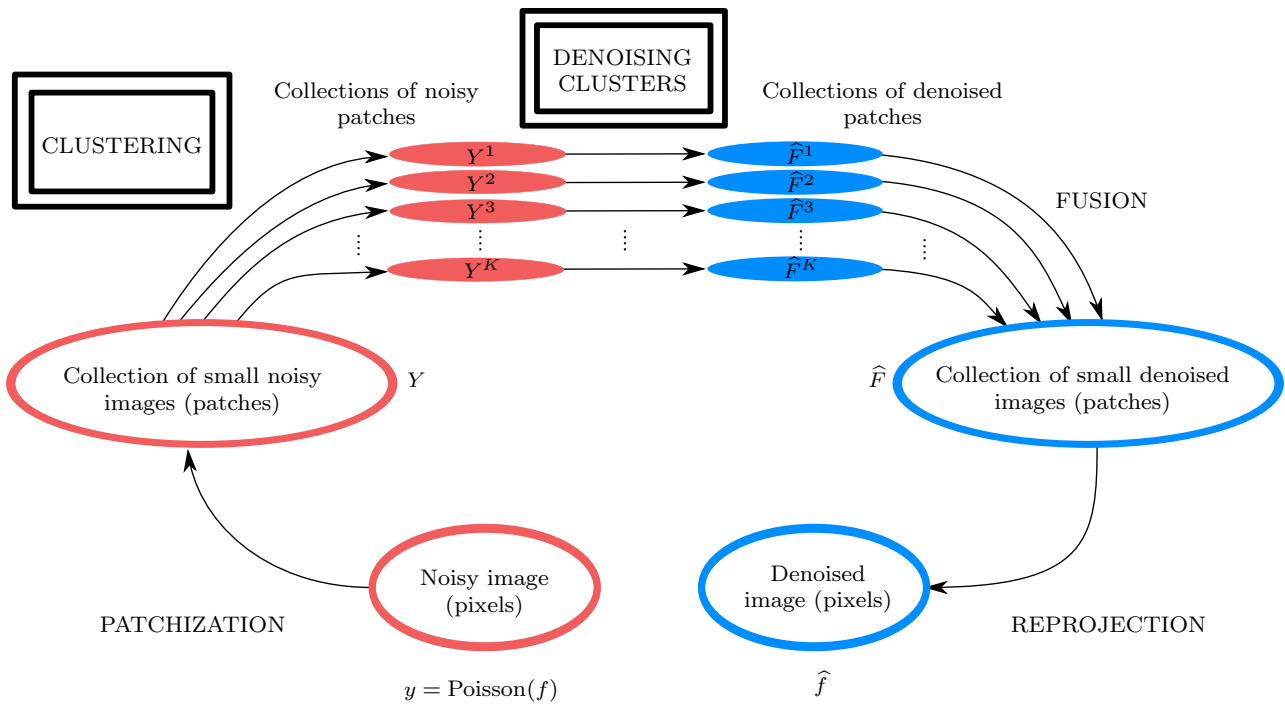
\begin{figure*}
\centering
\def\svgwidth{2\columnwidth}

\begingroup
  \makeatletter
  \providecommand\color[2][]{%
    \errmessage{(Inkscape) Color is used for the text in Inkscape, but the package 'color.sty' is not loaded}
    \renewcommand\color[2][]{}%
  }
  \providecommand\transparent[1]{%
    \errmessage{(Inkscape) Transparency is used (non-zero) for the text in Inkscape, but the package 'transparent.sty' is not loaded}
    \renewcommand\transparent[1]{}%
  }
  \providecommand\rotatebox[2]{#2}
  \ifx\svgwidth\undefined
    \setlength{\unitlength}{1497.59990234pt}
  \else
    \setlength{\unitlength}{\svgwidth}
  \fi
  \global\let\svgwidth\undefined
  \makeatother
  \begin{picture}(1,0.54680217)%
    \put(0,0){\includegraphics[width=\unitlength]{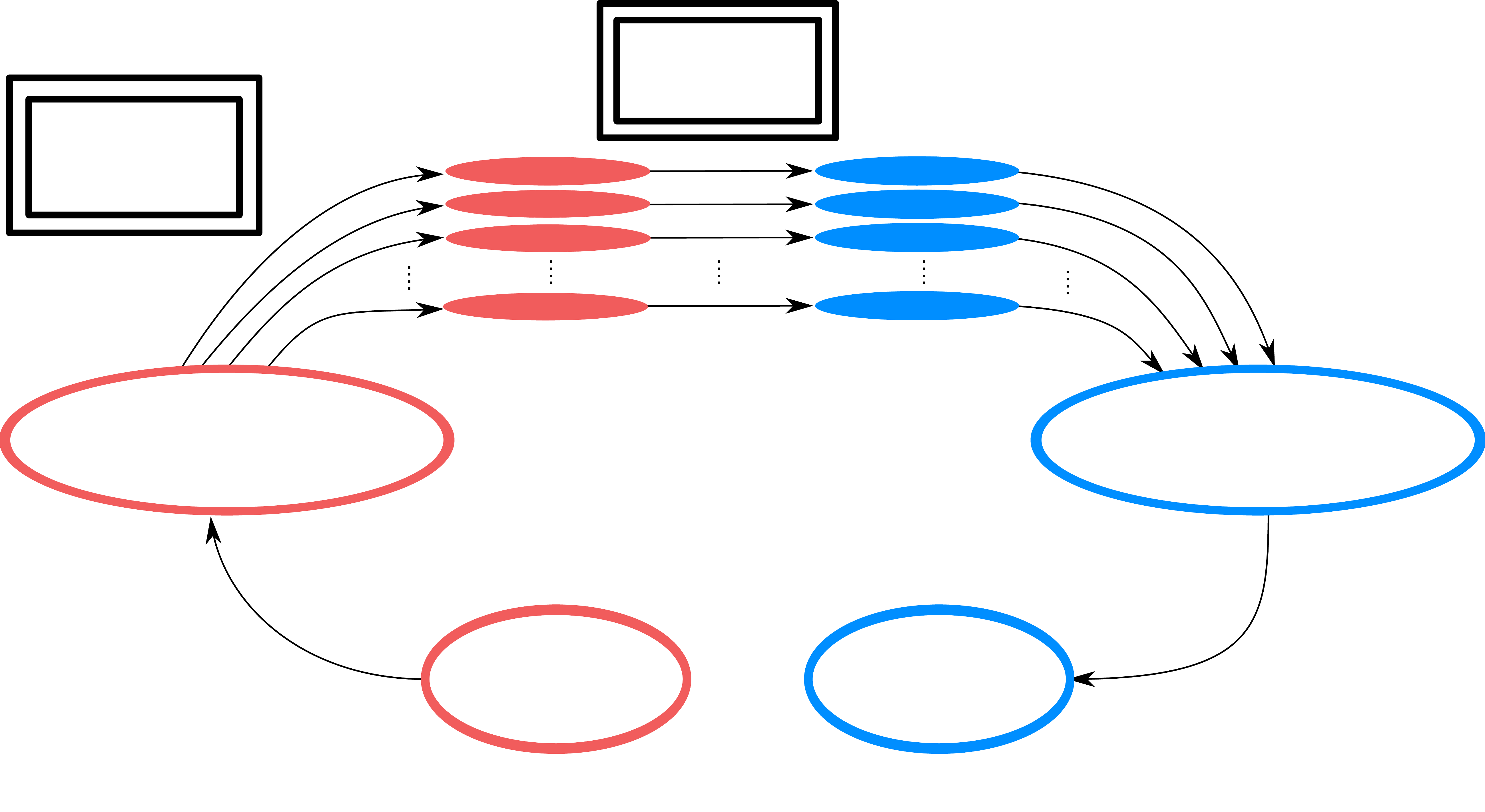}}%
    \put(0.37210023,0.4237047){\color[rgb]{0,0,0}\makebox(0,0)[b]{\smash{$Y^1$}}}%
    \put(0.37210023,0.40233719){\color[rgb]{0,0,0}\makebox(0,0)[b]{\smash{$Y^2$}}}%
    \put(0.85790604,0.38930801){\color[rgb]{0,0,0}\makebox(0,0)[b]{\smash{FUSION}}}%
    \put(0.62210022,0.4237047){\color[rgb]{0,0,0}\makebox(0,0)[b]{\smash{$\wh{F}^1$}}}%
    \put(0.31382832,0.2473803){\color[rgb]{0,0,0}\makebox(0,0)[lb]{\smash{$Y$}}}%
    \put(0.63239208,0.00679336){\color[rgb]{0,0,0}\makebox(0,0)[lb]{\smash{$\wh{f}$}}}%
    \put(0.67519996,0.24751802){\color[rgb]{0,0,0}\makebox(0,0)[lb]{\smash{$\wh{F}$}}}%
    \put(0.31125007,0.00355572){\color[rgb]{0,0,0}\makebox(0,0)[lb]{\smash{$y=\Poisson(f)$}}}%
    \put(0.01942262,0.44416063){\color[rgb]{0,0,0}\makebox(0,0)[lt]{\begin{minipage}{0.14173295\unitlength}\raggedright \centering CLUSTERING\end{minipage}}}%
    \put(0.41348338,0.51420166){\color[rgb]{0,0,0}\makebox(0,0)[lt]{\begin{minipage}{0.13784975\unitlength}\raggedright \centering DENOISING CLUSTERS\end{minipage}}}%
    \put(0.57931496,0.48820269){\color[rgb]{0,0,0}\makebox(0,0)[lt]{\begin{minipage}{0.21368613\unitlength}\raggedright \centering Collections of denoised patches\end{minipage}}}%
    \put(0.18282867,0.48820269){\color[rgb]{0,0,0}\makebox(0,0)[lt]{\begin{minipage}{0.21368613\unitlength}\raggedright \centering Collections of noisy patches\end{minipage}}}%
    \put(0.03234415,0.06660936){\color[rgb]{0,0,0}\makebox(0,0)[lt]{\begin{minipage}{0.2266368\unitlength}\raggedright \centering PATCHIZATION\end{minipage}}}%
    \put(0.75001268,0.06660936){\color[rgb]{0,0,0}\makebox(0,0)[lt]{\begin{minipage}{0.2266368\unitlength}\raggedright \centering REPROJECTION\end{minipage}}}%
    \put(0.01846764,0.26280314){\color[rgb]{0,0,0}\makebox(0,0)[lt]{\begin{minipage}{0.2541209\unitlength}\raggedright \centering Collection of small noisy images (patches) \end{minipage}}}%
    \put(0.71955586,0.26387151){\color[rgb]{0,0,0}\makebox(0,0)[lt]{\begin{minipage}{0.24661485\unitlength}\raggedright \centering Collection of small denoised images (patches) \end{minipage}}}%
    \put(0.3147131,0.10338169){\color[rgb]{0,0,0}\makebox(0,0)[lt]{\begin{minipage}{0.11902993\unitlength}\raggedright \centering Noisy image (pixels)\end{minipage}}}%
    \put(0.57229429,0.10706676){\color[rgb]{0,0,0}\makebox(0,0)[lt]{\begin{minipage}{0.12661253\unitlength}\raggedright \centering Denoised image (pixels)\end{minipage}}}%
    \put(0.3728633,0.33375245){\color[rgb]{0,0,0}\makebox(0,0)[b]{\smash{$Y^K$}}}%
    \put(0.37210023,0.37990129){\color[rgb]{0,0,0}\makebox(0,0)[b]{\smash{$Y^3$}}}%
    \put(0.62210022,0.37990129){\color[rgb]{0,0,0}\makebox(0,0)[b]{\smash{$\wh{F}^3$}}}%
    \put(0.62286329,0.33375245){\color[rgb]{0,0,0}\makebox(0,0)[b]{\smash{$\wh{F}^K$}}}%
    \put(0.62210022,0.40126881){\color[rgb]{0,0,0}\makebox(0,0)[b]{\smash{$\wh{F}^2$}}}%
  \end{picture}%
\endgroup

 \caption{\small Visual summary of our  denoising method. 
In this work we mainly focus on the two highlighted points of the figure: \textbf{clustering }in the context 
of very photon-limited
    data, and  specific \textbf{denoising} method for each cluster.
}
  \label{fig:patch_pca_architecture}
\end{figure*}

\section{Algorithmic details}\label{sec:algorithm}

We now present the practical implementation of our method, for the two variants that 
are the  Poisson NLPCA and the Poisson NLSPCA.





\subsection{Initialization}

We initialize the dictionary at random, drawing the entries from a
standard normal distribution, that we then normalize to have a unit
Euclidean norm. This is equivalent to generating the atoms uniformly
at random from the Euclidean unit sphere.  As a rule of thumb, we also
constrain the first atom (or axis) to be initialized as a constant
vector.  However, this constraint is not enforced during the
iterations, so this property can be lost after few steps.

\subsection{Stopping criterion and conditioning number}

Many methods are proposed in \cite{Wright_Nowak_Figueiredo09} for the stopping criterion.
Here we have used  a criterion  based on the relative 
change in the objective function $L^{\pen}(U,V)$ defined in Eq. \eqref{eq:poisson_constrained_KL}.
This means that we iterate the alternating updates in the algorithm 
as long $\|\exp (U_{t} V_{t})- \exp (U_{t+1} V_{t+1})\|^2/\|\exp (U_{t} V_{t})\|^2 \leq
\varepsilon_{\stoping} $ for some (small) real number $
\varepsilon_{\stoping}$. 

For numerical stability we have added a Tikhonov (or ridge) regularization term. Thus, we
have substituted $ V_t D_{i} V_t^\top$ in Eq.~\eqref{eq:update_row_U_poisson} with $(V_t D_{i} V_t^\top+\varepsilon_{\cond} I_{\ell})$ 
and $(U_t ^\top E_{j} U_t )$ in Eq.~\eqref{eq:update_row_V_poisson} 
 with $(U_t ^\top E_{j} U_t )+\varepsilon_{\cond} I_{\ell})$. For the NLSPCA version the $\varepsilon_{\cond}$ parameter
is only used to update the dictionary in Eq.~\eqref{eq:update_row_V_poisson}, since the regularization on the
coefficients is provided by Eq.~\eqref{eq:update_row_U_poisson_l1}.

\subsection{Reprojections}

Once the whole collection of patches is denoised, it remains to reproject the information onto the pixels. 
Among various solutions proposed in the literature (see for instance \cite{Salmon_Strozecki12} and 
\cite{Dabov_Foi_Katkovnik_Egiazarian07}) the most popular, the one we use in our experiments, is to uniformly 
average all the estimates provided by the patches containing the given pixel.

\subsection{Binning-interpolating}
Following a suggestion of an anonymous reviewer, we have also investigated the following ``binned''
variant of our method:
\begin{enumerate}
 \item aggregate the noisy Poisson pixels into small (for instance $3\times 3$) bins, resulting in a smaller Poisson image with lower resolution but higher counts per pixel;
\item  denoise this binned image using our proposed method;
\item  enlarge the denoised image to the original size using (for instance bilinear) interpolation.
\end{enumerate}
Indeed, in the extreme noise level case we have considered, this approach significantly reduces computation time, and for some images it yields a significant performance increase.
The binning process  allows us to implicitly use larger patches, without facing challenging
memory and computation time issues.
Of course, such a scheme could be applied to any method dealing with low photon counts, and we provide a comparison with the BM3D method (the best overall  competing method) in the experiments section.

%

\section{Experiments}\label{sec:experiments}
We have conducted experiments both on simulated and on real data, on grayscale images (2D) 
and on spectral images (3D). We summarize our
results in the following, both with visual results and performance metrics.

\subsection{Simulated 2D data}
\begin{figure}[h!]
\centering
\begin{tabular}{@{} c@{ }c@{ }c@{ }c @{}} 
\includegraphics[width=0.231\linewidth]{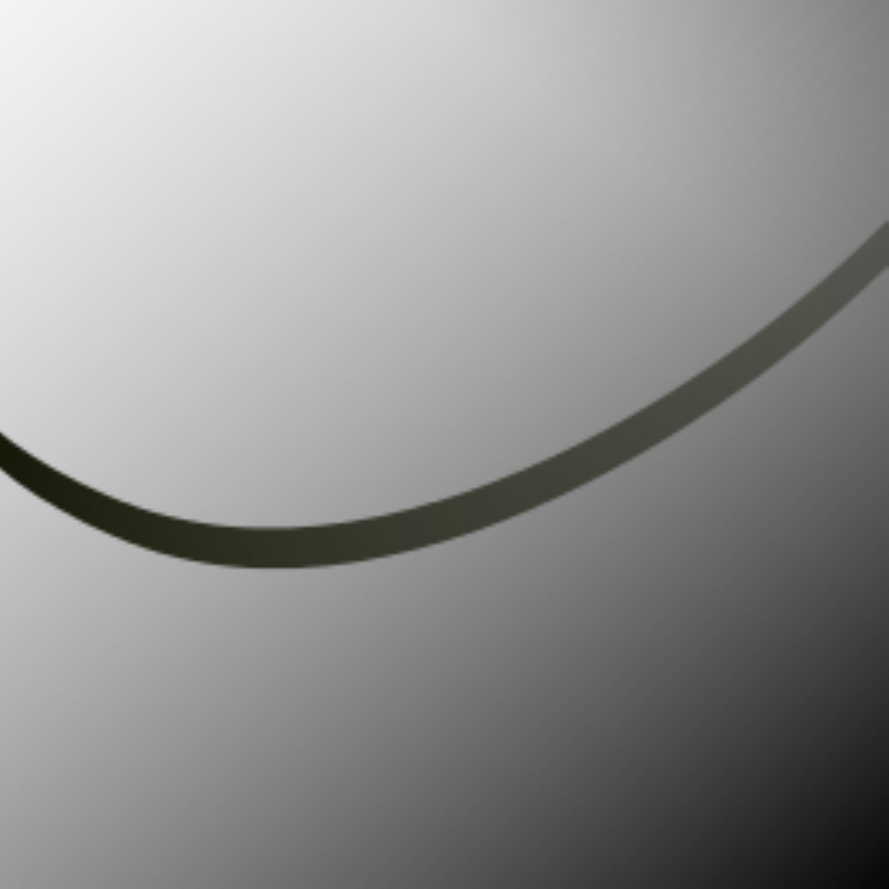} &
\includegraphics[width=0.231\linewidth]{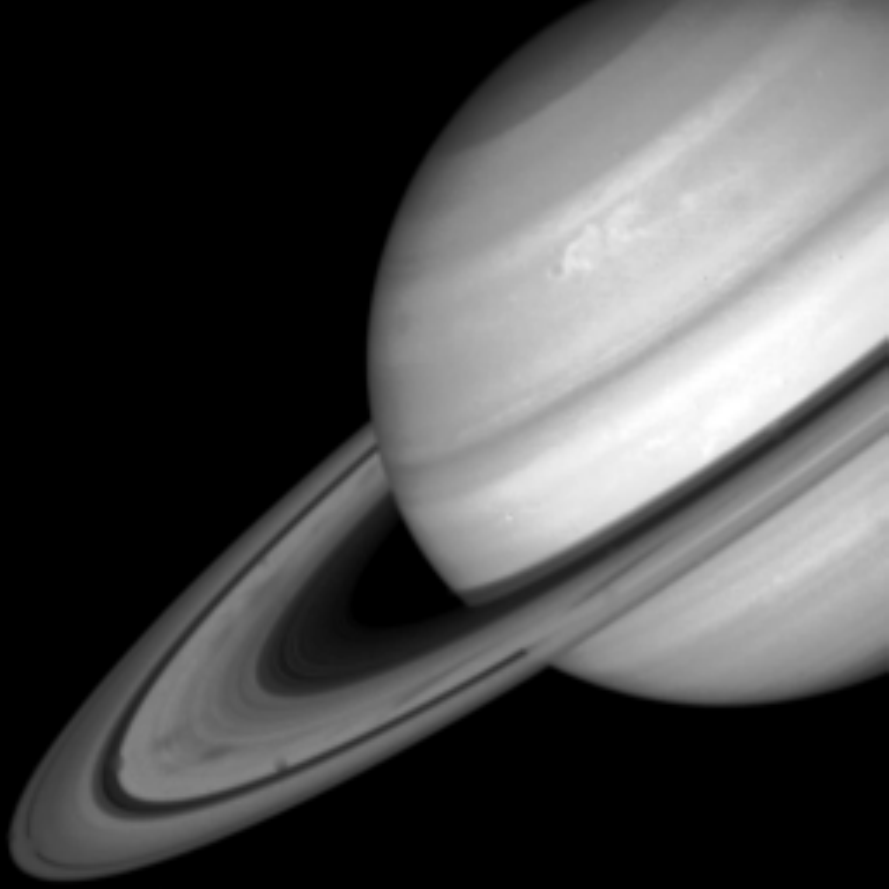} &
\includegraphics[width=0.231\linewidth]{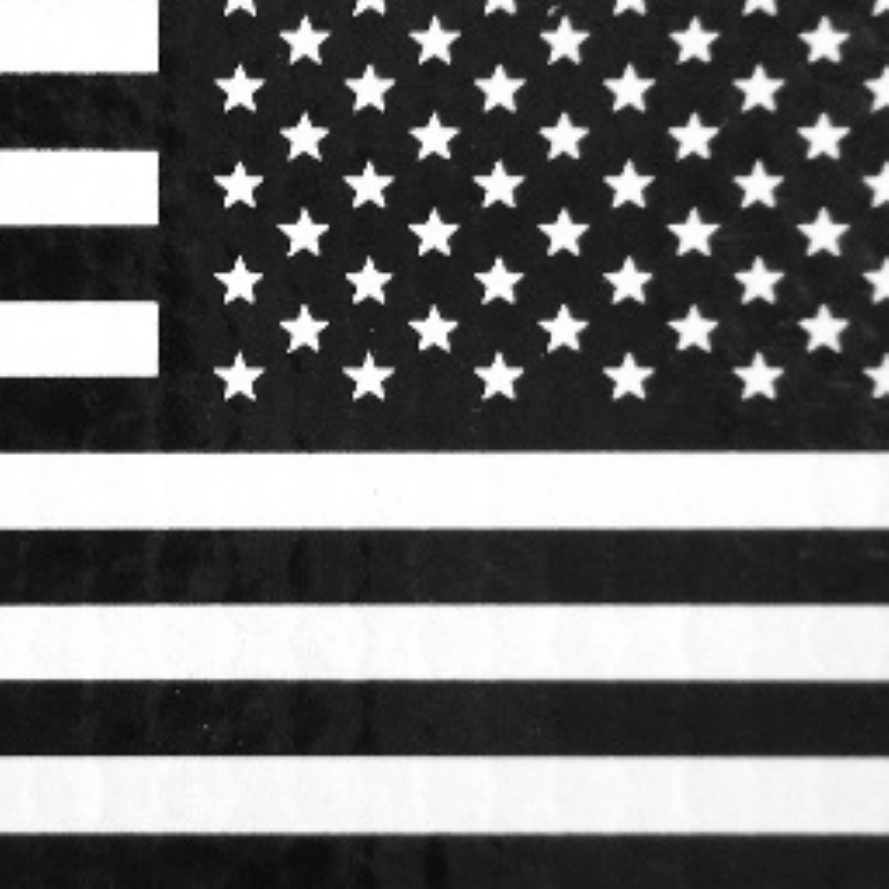} &
\includegraphics[width=0.231\linewidth]{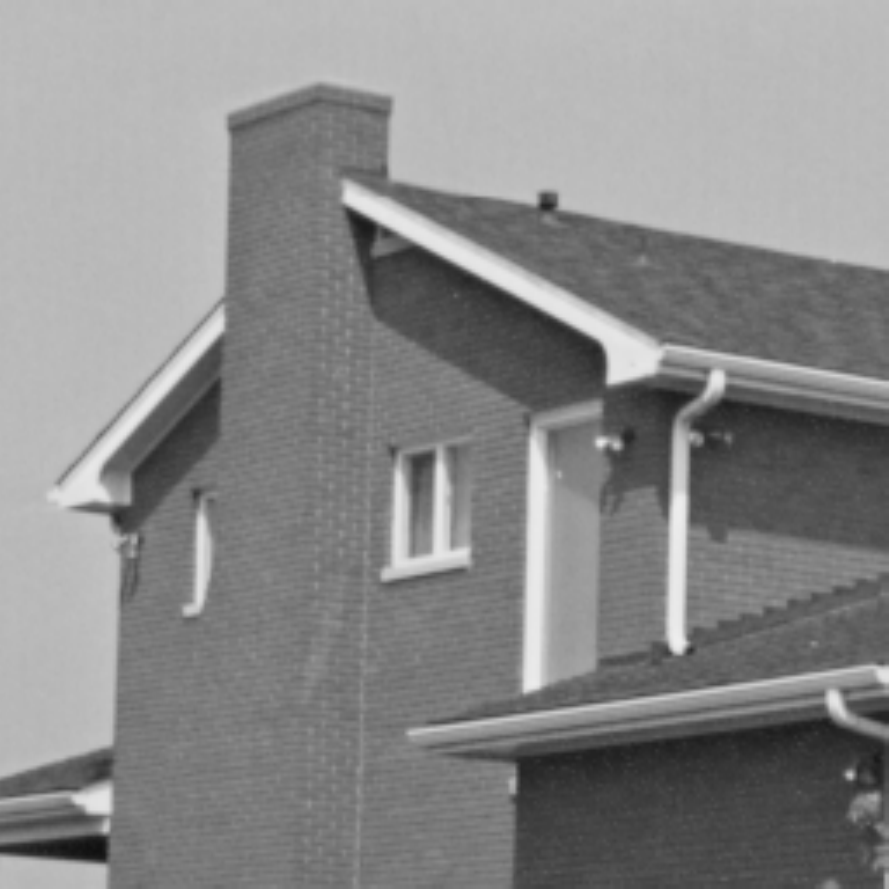} \\
Swoosh & Saturn & Flag & House \\
\includegraphics[width=0.231\linewidth]{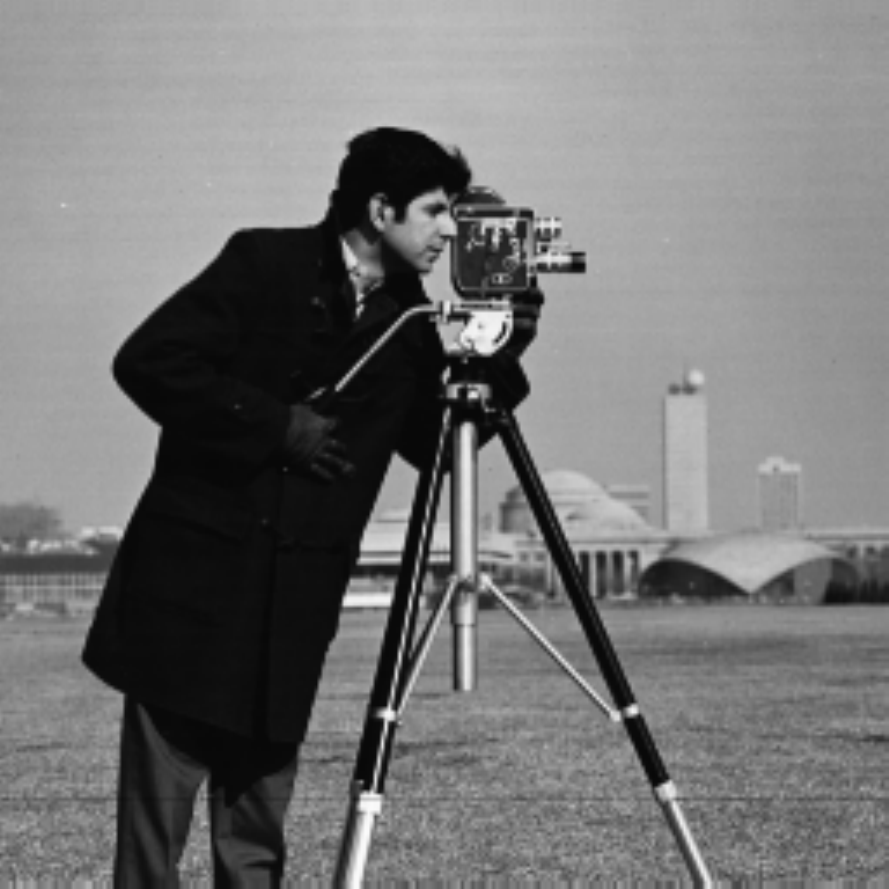} &
\includegraphics[width=0.231\linewidth]{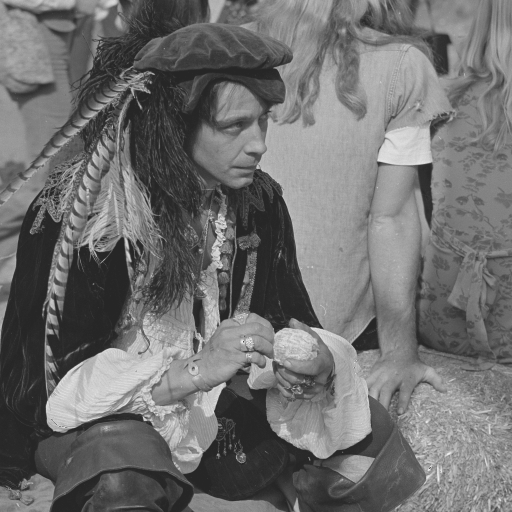} &
\includegraphics[width=0.231\linewidth]{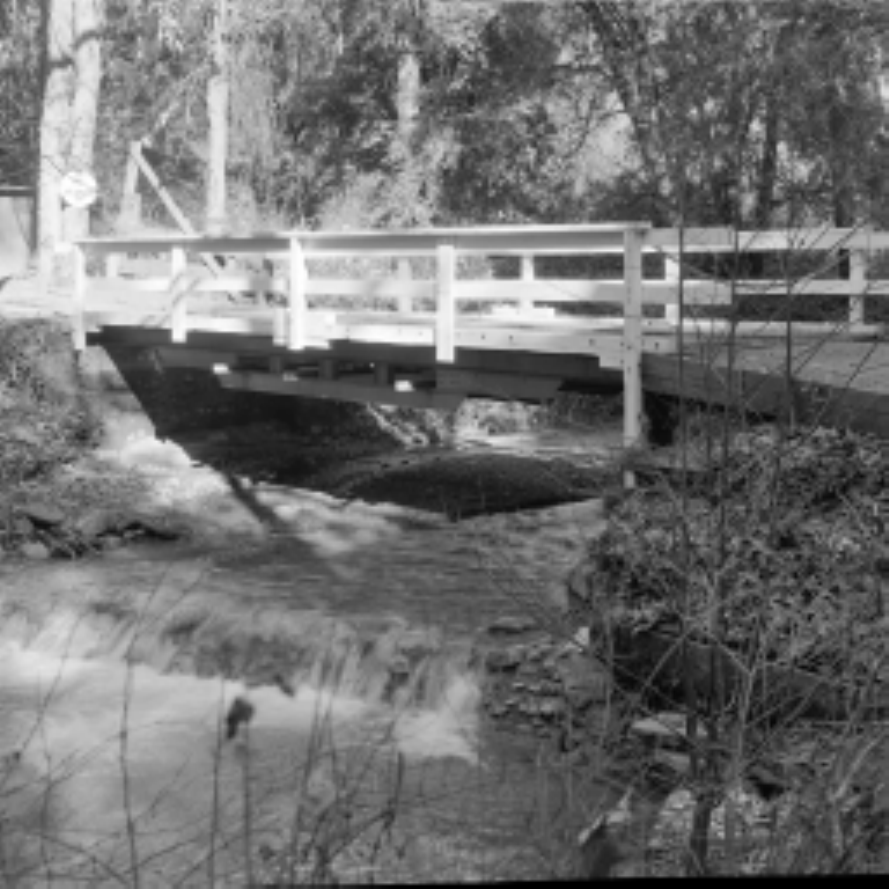} &
\includegraphics[width=0.231\linewidth]{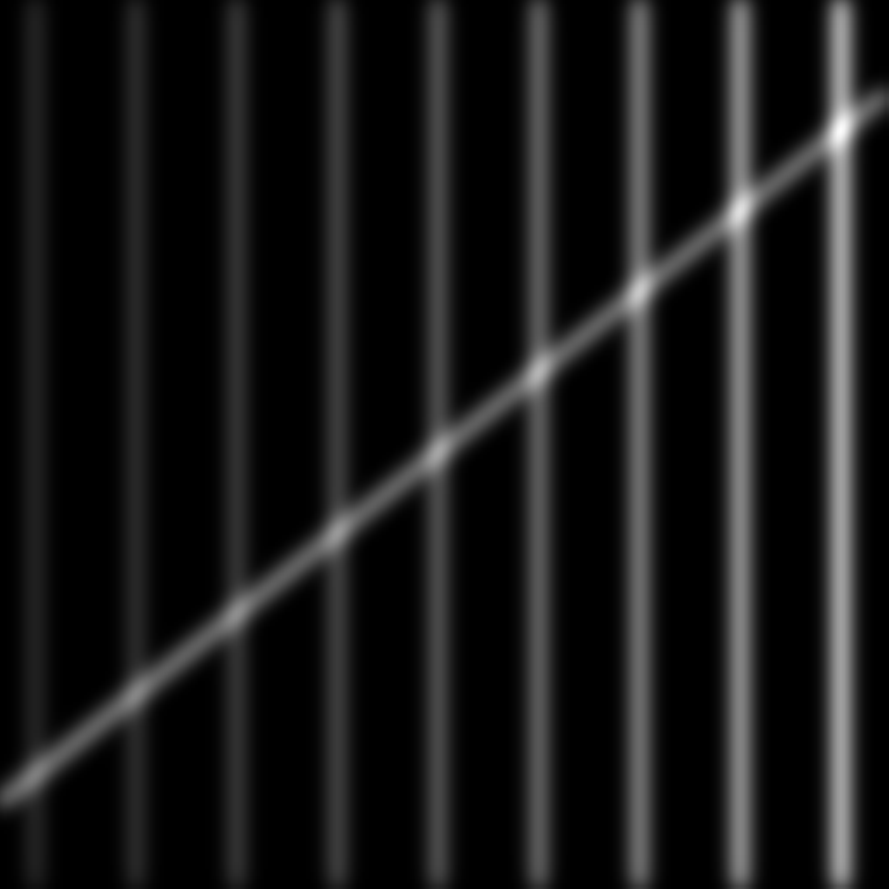} \\
Cameraman & Man & Bridge & Ridges 
\end{tabular}
\caption{Original images used for our simulations.}
\label{fig:original_noisy_images}
\end{figure}

We have first conducted comparisons of our method and several competing algorithms on simulated data. The images
we have used in the simulations are presented in Fig.~\ref{fig:original_noisy_images}.
We have considered the same noise level for the Saturn image (\lcf Fig.~\ref{fig:Saturn_peak=02}) 
as in \cite{Willett06}, where one can find
extensive comparisons with a variety of multiscale methods 
\cite{Kolaczyk99,Willett_Nowak03,Kolaczyk_Nowak04}.
%

In terms of PSNR, defined in the classical way (for 8-bit images)
\begin{align}
\PSNR(\wh{f}, f) &=
  10 \log_{10} \frac{255^2}{\frac{1}{M} { \displaystyle \sum_{i} (\wh{f}_i -
f_i)^2}},
\end{align}
our method globally improves upon other state-of-the-art methods such
as Poisson-NLM \cite{Deledalle_Denis_Tupin10}, SAFIR
\cite{Boulanger_Kervrann_Bouthemy_Elbau_Sibarita_Salamero10}, and Poisson
Multiscale Partitioning (PMP) \cite{Willett_Nowak03} for the very low
light levels of interest.  Moreover, visual artifacts tend to be reduced by our
Poisson NLPCA and NLSPCA, with respect to the version using an
Anscombe transform and classical PCA (\lcf AnscombeNLPCA in
Figs.~\ref{fig:Saturn_peak=02} and \ref{fig:Flag_peak=01} for
instance).  See Section~\ref{sec:compare} for more
details on the methods used for comparison.

All our results for 2D and 3D images are provided for both 
the NLPCA and NLSPCA using (except otherwise stated)
the parameter values summarized in Table~\ref{tab:parameters}.
The step-size parameter $\alpha_t$ for the NL-SPCA method is chosen via a selection rule initialized with the Barzilai-Borwein choice, 
as described in \cite{Harmany_Marcia_Willett12}.

\begin{table}[tb]
  \centering
  \begin{tabular}{@{} clc}
  \toprule
  Parameter & Definition & Value\\
  \midrule
  $N$			 & patch size & $20 \times 20$ \\
  $\ell$ & approximation rank & 4 \\
  $K$ & clusters & 14 \\
  $N_{\iter}$ & iteration limit & 20 \\
  $\varepsilon_{{\rm stop}}$ & stopping tolerance & $10^{-1}$ \\
  $\varepsilon_{\cond}$ & conditioning parameter & $10^{-3}$ \\
  \multirow{2}{*}{$\lambda$} & $\ell_1$ regularization
  & \multirow{2}{*}{$70\sqrt{\frac{\log(M_k)}{n}}$} \\
   & (NL-SPCA only) & \\
  \bottomrule
  \end{tabular}
  \caption{Parameter settings used in the proposed method. Note: $M_k$
  		is the number of patches in the $k$th cluster as determined by the 
		Bregman hard clustering step.}
  \label{tab:parameters}
\end{table}


\subsection{Simulated 3D data}

In this section we have tested a generalization of our algorithm for spectral
images. We have thus considered the NASA AVIRIS (Airborne Visible/Infrared Imaging
Spectrometer) Moffett Field reflectance data set,
and we have kept a $256 \times 256\times 128$ sized portion of the total data cube.
For the simulation we have used the same noise level as in \cite{Krishnamurthy_Raginsky_Willett10} 
(the number of photons per voxel is 0.0387), so that comparison could be done with the results
presented in this paper. Moreover to ease comparison  with earlier work, the performance
has been measured in terms of mean absolute error (MAE), defined by 
\begin{equation}
 \MAE(\wh{f}, f)=\frac{\|\wh{f}-f\|_1}{\|f\|_1}.
\end{equation}

We have performed the clustering on the 2D image obtained by summing the
photons on the third (spectral) dimension, and using this clustering
for each 3D patch. This approach is particularly well suited for low
photons counts since with other approaches the clustering step can be of poor quality.
Our approach provides an illustration of the importance of taking into
account the correlations across the channels.  We have used non-square
patches since the spectral image intensity has different levels of homogeneity across
the spectral and spatial dimensions. We thus have considered elongated
patches with respect to the third dimension. In practice, the patch
size used for the results presented is $5\times5 \times 23$, the
number of clusters is $K=30$, and the order of approximation is
$\ell=2$.

For the noise level considered, our proposed algorithm outperforms the
other methods, BM4D \cite{Maggioni_Katkovnik_Egiazarian_Foi11} and PMP \cite{Krishnamurthy_Raginsky_Willett10},
both visually and in term of MAE (\lcf Fig.~\ref{fig:moffet_close_up}). Again, these competing methods are
described in Section~\ref{sec:compare}.

\subsection{Real 3D data }

We have also used our method to denoise some real noisy astronomical data.
The last image we have considered is based on thermal X-ray emissions of the 
youngest supernova explosion ever observed. It is the supernova remnant G1.9+0.3  (\symbol{64} NASA/CXC/SAO) 
in the Milky Way.  The study of such spectral
images can provide important information about the nature of elements present in the early stages
of supernova. We refer to \cite{Borkowski_Reynolds_Green_Hwang_Petre_Krishnamurthy_Willett10} for
deeper insights on the implications  for astronomical science.
This dataset has an average of 0.0137 photons per voxel.

For this image we have also used the 128 first spectral channels, so the data cube is also of size 
$256 \times 256 \times128$. 
Our method removes some of the spurious artifacts generated by the method proposed in~\cite{Krishnamurthy_Raginsky_Willett10} and the blurry artifacts in BM4D~\cite{Maggioni_Katkovnik_Egiazarian_Foi11}.

\subsection{Comparison with other methods}
\label{sec:compare}
\subsubsection{Classical PCA  with Anscombe transform}
The approximation of the variance provided by the Anscombe transform
is reasonably accurate for intensities of three or more (\lcf
Fig.~\ref{fig:Anscombe_variance} and also \cite{Makitalo_Foi12}
Fig.~1-b). In practice this is also the regime where a well-optimized
method for Gaussian noise might be applied successfully using this
transform and the inverse provided in \cite{Makitalo_Foi11}.
\begin{figure}[t]
\centering
\subfigure{\includegraphics[width=0.8125\linewidth]{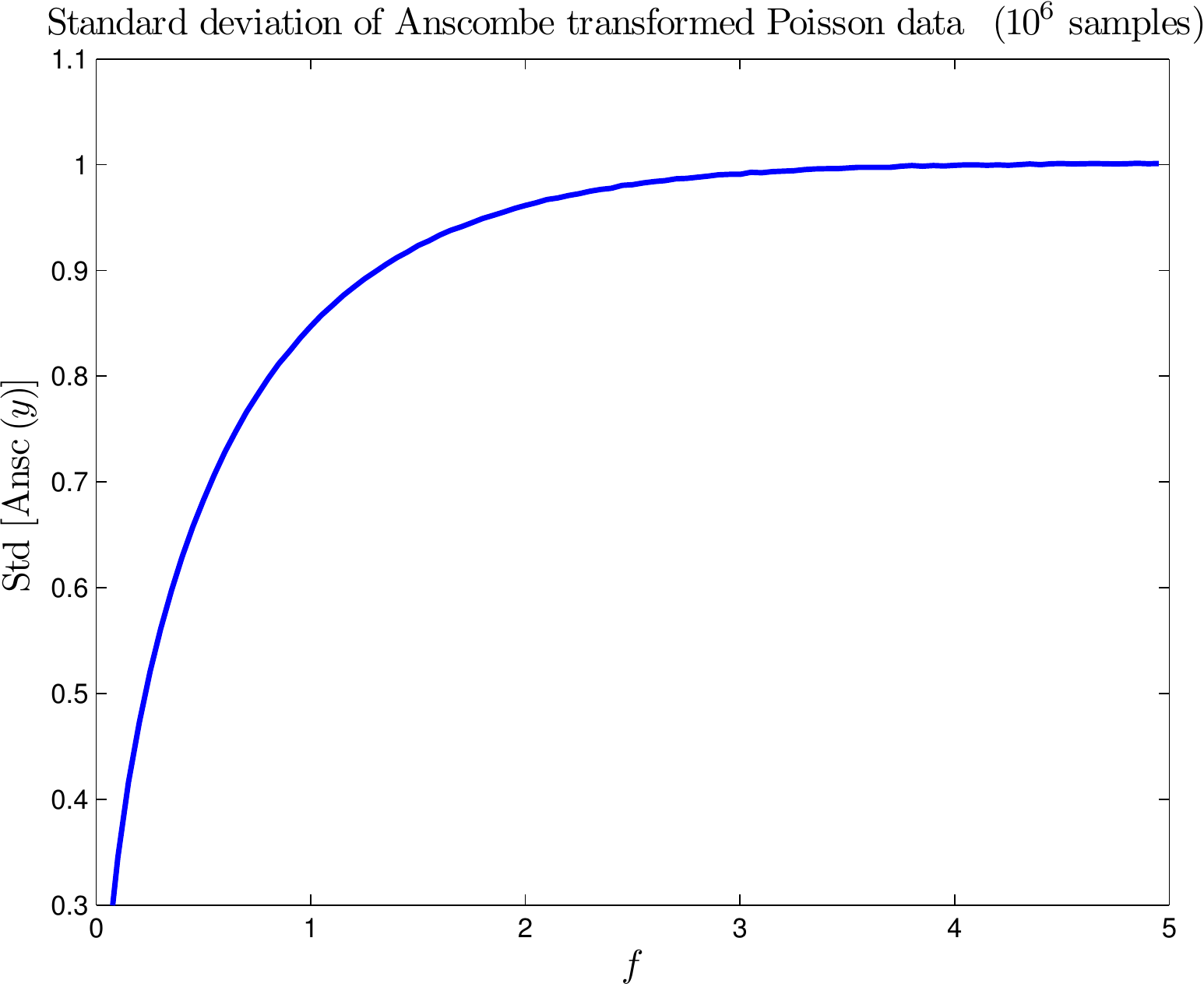}}
\caption{Standard deviation approximation of some simulated Poisson
  data, after performing the Anscombe transform (Ansc). For each true parameter $f$, $10^6$
Poisson realizations where drawn and the corresponding standard deviation is reported. }
\label{fig:Anscombe_variance}
\end{figure}

To compare the importance of fully taking advantage of the Poisson
model and not using the Anscombe transform, we have derived another algorithm, 
analogous to our Poisson NLPCA method but using Bregman divergences associated 
with the natural parameter of a Gaussian random variable instead of Poisson. 
It corresponds to an implementation similar to
the classical power method for computing PCA
\cite{Collins_Dasgupta_Schapire02}.  The function $L$ to be optimized
in \eqref{eq:original_poisson_KL} is simply replaced by the square loss
$\tilde{L}$,\vspace{-0.2cm}
\begin{equation}\label{eq:original_gaussian_KL}
\tilde{L}(U,V)= \sum_{i=1}^M \sum_{j=1}^N  \left( (UV)_{i,j} -Y_{i,j} \right)^2 \, . 
\end{equation}
For the Gaussian case, the following update equations are substituted for \eqref{eq:update_row_U_poisson} and \eqref{eq:update_row_V_poisson}

\begin{equation}\displaystyle \label{eq:update_row_U_gaussian}
U_{t+1,i,:}=U_{t,i,:} -((U_tV_t)_{i,:} -Y_{i,:})V_t^\top (V_t V_t^\top)^{-1} ~,\\
\end{equation}
and
\begin{multline}
V_{t+1,:,j}=\\
V_{t,:,j} -  (U_{t+1} ^\top U_{t+1} )^{-1} U_{t+1}^\top \left((U_{t+1}V_t)_{:,j} -Y_{:,j}\right) ~.
\end{multline}

An illustration of the improvement due to our direct modeling of Poisson noise instead of a simpler 
Anscombe (Gaussian) NLPCA approach is shown 
in our previous work \cite{Salmon_Deledalle_Willett_Harmany12} and the below simulation results. 
The gap is most noticeable at low signal-to-noise ratios, and high-frequency artifacts are more likely 
to appear when using the Anscombe transform.  To invert the Anscombe transform we have considered 
the function provided 
by \cite{Makitalo_Foi11}, and available at \url{http://www.cs.tut.fi/~foi/invansc/}.
This slightly improves the usual (closed form) inverse transformation, and in our work it is used for all 
the methods using the Anscombe transform (referred to as Anscombe-NLPCA in our experiments).

\subsubsection{Other methods }
We compare our method with other recent algorithms designed for retrieval of Poisson corrupted images. In
the case of 2D images we have compared with:
\begin{itemize}

\item NLBayes \cite{Lebrun_Colom_Buades_Morel12} using Anscombe transform and the refined inverse transform proposed in \cite{Makitalo_Foi11}.
\item SAFIR \cite{Kervrann_Boulanger06,Boulanger_Kervrann_Bouthemy_Elbau_Sibarita_Salamero10}, 
using Anscombe transform and the refined inverse transform proposed in \cite{Makitalo_Foi11}.
\item Poisson multiscale partitioning (PMP), introduced by Willett and
  Nowak \cite{Willett_Nowak03,Willett_Nowak04} using full cycle
  spinning.
We use the haarTIApprox function as available at \url{http://people.ee.duke.edu/~willett}.
\item BM3D \cite{Makitalo_Foi11} using Anscombe transform with a refined inverse transform. 
The online code is available at \url{http://www.cs.tut.fi/~foi/invansc/} and we used the 
default parameters provided by the authors. The version with binning and interpolation relies on
$3\times 3$ bins and bilinear interpolation.
\end{itemize}

In the case of spectral images we have compared our proposed method with 
\begin{itemize}
\item  BM4D \cite{Maggioni_Katkovnik_Egiazarian_Foi11} using 
the inverse Anscombe \cite{Makitalo_Foi11} already mentioned. We set the patch size to $4 \times 4 \times 16$,
since the patch length has to be dyadic for this algorithm.
\item Poisson multiscale partition (PMP for 3D images)
  \cite{Krishnamurthy_Raginsky_Willett10}, adapting the haarTIApprox
  algorithm to the case of spectral images. As in the reference
  mentioned, we have considered cycle spinning with 2000 shifts.
\end{itemize}

  For visual inspection of the qualitative performance of each approach,
  the results are displayed on Fig.~\ref{fig:Ridges_peak=01}-\ref{fig:chandra}. Quantitative performance  in terms
  of PSNR are given in Tab.~\ref{tab:optimal_bandwidth}.



\begin{figure*}[hbt!]
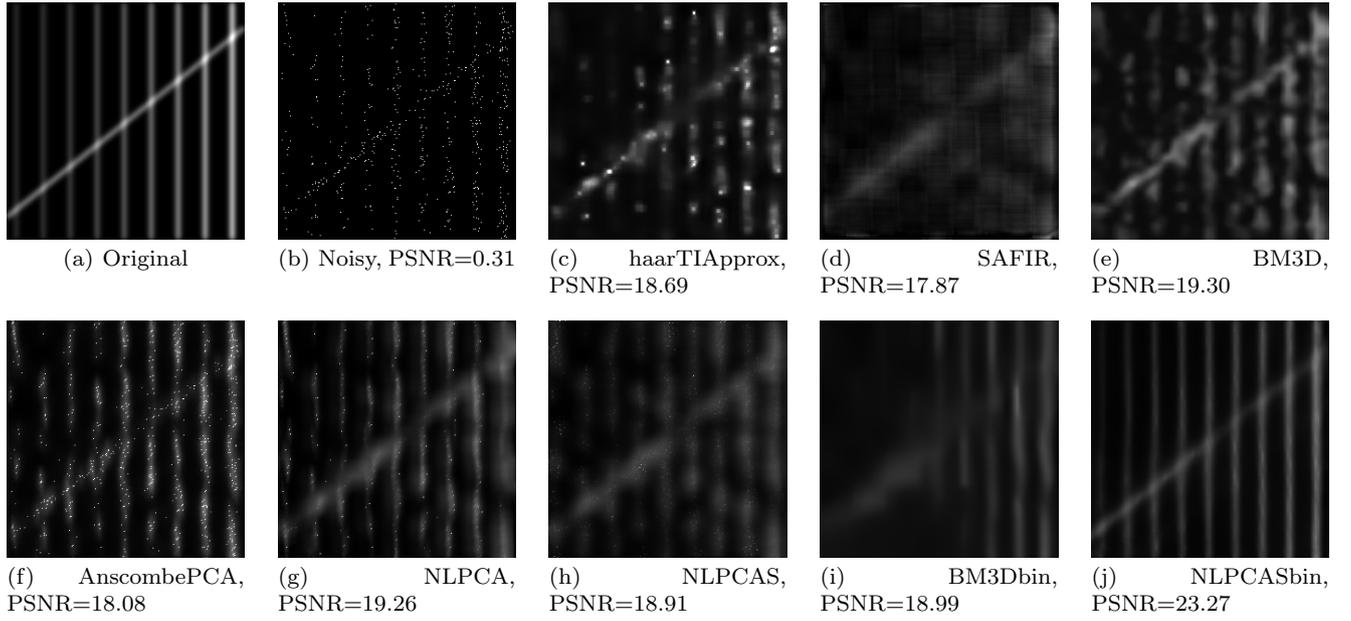

\centering
\input{Im_Original_01_Ridges_legend}  \subfigure[\currentcaption]{\includegraphics[width=0.1798\linewidth]{\currentname}}\hfill
\input{Im_Noisy_01_Ridges_legend}  \subfigure[\currentcaption]{\includegraphics[width=0.1798\linewidth]{\currentname}}\hfill
\input{Im_haarTIApprox_01_Ridges_legend}  \subfigure[\currentcaption]{\includegraphics[width=0.1798\linewidth]{\currentname}}\hfill
\input{Im_SAFIR_01_Ridges_legend}  \subfigure[\currentcaption]{\includegraphics[width=0.1798\linewidth]{\currentname}}\hfill
\input{Im_BM3D_01_Ridges_legend}  \subfigure[\currentcaption]{\includegraphics[width=0.1798\linewidth]{\currentname}}\hfill
\input{Im_AnscombePCA_01_Ridges_legend}  \subfigure[\currentcaption]{\includegraphics[width=0.1798\linewidth]{\currentname}}\hfill
\input{Im_NLPCA_01_Ridges_legend}  \subfigure[\currentcaption]{\includegraphics[width=0.1798\linewidth]{\currentname}}\hfill
\input{Im_NLPCAS_01_Ridges_legend}  \subfigure[\currentcaption]{\includegraphics[width=0.1798\linewidth]{\currentname}}\hfill
\input{Im_BM3Dbin_01_Ridges_legend}  \subfigure[\currentcaption]{\includegraphics[width=0.1798\linewidth]{\currentname}}\hfill
\input{Im_NLPCASbin_01_Ridges_legend}  \subfigure[\currentcaption]{\includegraphics[width=0.1798\linewidth]{\currentname}}\hfill
\caption{Toy cartoon image (Ridges) corrupted with Poisson noise with  Peak = 0.1.}
\label{fig:Ridges_peak=01}
\end{figure*}

\begin{figure*}[hbt!]
\centering
\input{Im_Original_1_Ridges_legend}  \subfigure[\currentcaption]{\includegraphics[width=0.1798\linewidth]{\currentname}}\hfill
\input{Im_Noisy_1_Ridges_legend}  \subfigure[\currentcaption]{\includegraphics[width=0.1798\linewidth]{\currentname}}\hfill
\input{Im_haarTIApprox_1_Ridges_legend}  \subfigure[\currentcaption]{\includegraphics[width=0.1798\linewidth]{\currentname}}\hfill
\input{Im_SAFIR_1_Ridges_legend}  \subfigure[\currentcaption]{\includegraphics[width=0.1798\linewidth]{\currentname}}\hfill
\input{Im_BM3D_1_Ridges_legend}  \subfigure[\currentcaption]{\includegraphics[width=0.1798\linewidth]{\currentname}}\hfill
\input{Im_AnscombePCA_1_Ridges_legend}  \subfigure[\currentcaption]{\includegraphics[width=0.1798\linewidth]{\currentname}}\hfill
\input{Im_NLPCA_1_Ridges_legend}  \subfigure[\currentcaption]{\includegraphics[width=0.1798\linewidth]{\currentname}}\hfill
\input{Im_NLPCAS_1_Ridges_legend}  \subfigure[\currentcaption]{\includegraphics[width=0.1798\linewidth]{\currentname}}\hfill
\input{Im_BM3Dbin_1_Ridges_legend}  \subfigure[\currentcaption]{\includegraphics[width=0.1798\linewidth]{\currentname}}\hfill
\input{Im_NLPCASbin_1_Ridges_legend}  \subfigure[\currentcaption]{\includegraphics[width=0.1798\linewidth]{\currentname}}\hfill
\caption{Toy cartoon image (Ridges) corrupted with Poisson noise with  Peak = 1.  }
\label{fig:Ridges_peak=1}
\end{figure*}

\begin{figure*}[hbt!]
\centering
\input{Im_Original_01_Flag_legend}  \subfigure[\currentcaption]{\includegraphics[width=0.1798\linewidth]{\currentname}}\hfill
\input{Im_Noisy_01_Flag_legend}  \subfigure[\currentcaption]{\includegraphics[width=0.1798\linewidth]{\currentname}}\hfill
\input{Im_haarTIApprox_01_Flag_legend}  \subfigure[\currentcaption]{\includegraphics[width=0.1798\linewidth]{\currentname}}\hfill
\input{Im_SAFIR_01_Flag_legend}  \subfigure[\currentcaption]{\includegraphics[width=0.1798\linewidth]{\currentname}}\hfill
\input{Im_BM3D_01_Flag_legend}  \subfigure[\currentcaption]{\includegraphics[width=0.1798\linewidth]{\currentname}}\hfill
\input{Im_AnscombePCA_01_Flag_legend}  \subfigure[\currentcaption]{\includegraphics[width=0.1798\linewidth]{\currentname}}\hfill
\input{Im_NLPCA_01_Flag_legend}  \subfigure[\currentcaption]{\includegraphics[width=0.1798\linewidth]{\currentname}}\hfill
\input{Im_NLPCAS_01_Flag_legend}  \subfigure[\currentcaption]{\includegraphics[width=0.1798\linewidth]{\currentname}}\hfill
\input{Im_BM3Dbin_01_Flag_legend}  \subfigure[\currentcaption]{\includegraphics[width=0.1798\linewidth]{\currentname}}\hfill
\input{Im_NLPCASbin_01_Flag_legend}  \subfigure[\currentcaption]{\includegraphics[width=0.1798\linewidth]{\currentname}}\hfill
\caption{Toy cartoon image (Flag) corrupted with Poisson noise with  Peak = 0.1.}
\label{fig:Flag_peak=01}
\end{figure*}

\begin{figure*}[hbt!]
\centering
\input{Im_Original_1_Flag_legend}  \subfigure[\currentcaption]{\includegraphics[width=0.1798\linewidth]{\currentname}}\hfill
\input{Im_Noisy_1_Flag_legend}  \subfigure[\currentcaption]{\includegraphics[width=0.1798\linewidth]{\currentname}}\hfill
\input{Im_haarTIApprox_1_Flag_legend}  \subfigure[\currentcaption]{\includegraphics[width=0.1798\linewidth]{\currentname}}\hfill
\input{Im_SAFIR_1_Flag_legend}  \subfigure[\currentcaption]{\includegraphics[width=0.1798\linewidth]{\currentname}}\hfill
\input{Im_BM3D_1_Flag_legend}  \subfigure[\currentcaption]{\includegraphics[width=0.1798\linewidth]{\currentname}}\hfill
\input{Im_AnscombePCA_1_Flag_legend}  \subfigure[\currentcaption]{\includegraphics[width=0.1798\linewidth]{\currentname}}\hfill
\input{Im_NLPCA_1_Flag_legend}  \subfigure[\currentcaption]{\includegraphics[width=0.1798\linewidth]{\currentname}}\hfill
\input{Im_NLPCAS_1_Flag_legend}  \subfigure[\currentcaption]{\includegraphics[width=0.1798\linewidth]{\currentname}}\hfill
\input{Im_BM3Dbin_1_Flag_legend}  \subfigure[\currentcaption]{\includegraphics[width=0.1798\linewidth]{\currentname}}\hfill
\input{Im_NLPCASbin_1_Flag_legend}  \subfigure[\currentcaption]{\includegraphics[width=0.1798\linewidth]{\currentname}}\hfill
\caption{Toy cartoon image (Flag) corrupted with Poisson noise with  Peak = 1.  }
\label{fig:Flag_peak=1}
\end{figure*}

\begin{figure*}[hbt!]
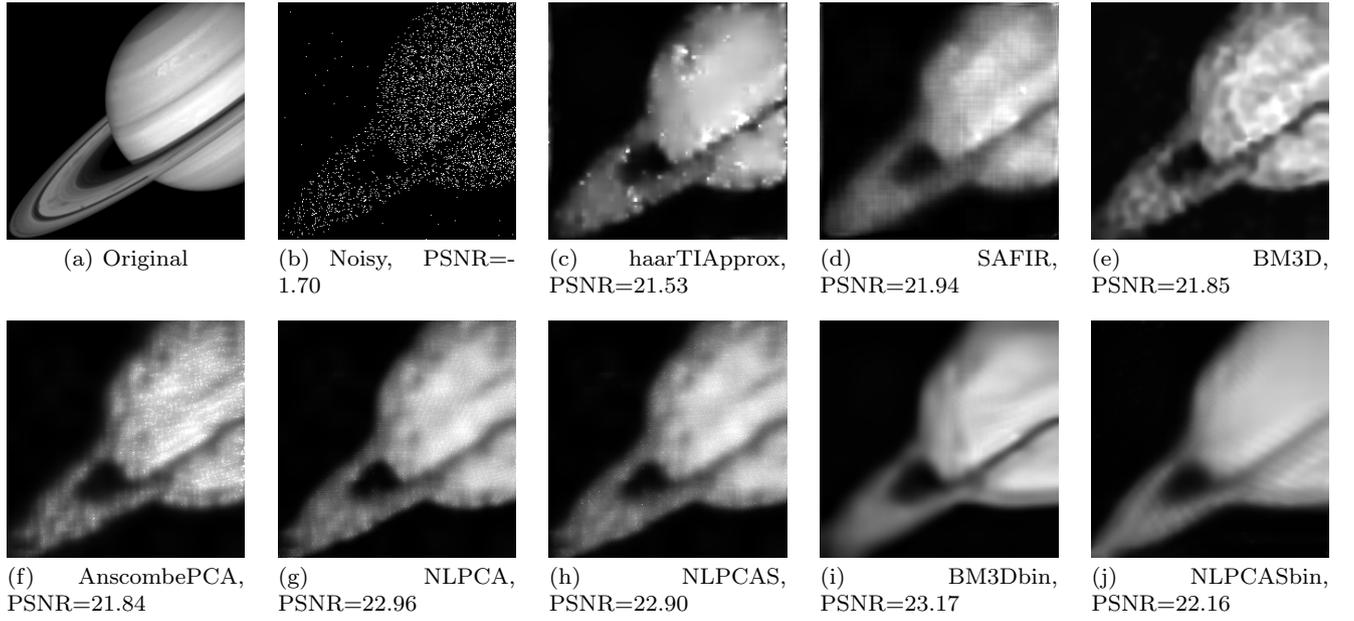

\centering
\input{Im_Original_02_Saturn_legend}  \subfigure[\currentcaption]{\includegraphics[width=0.1798\linewidth]{\currentname}}\hfill
\input{Im_Noisy_02_Saturn_legend}  \subfigure[\currentcaption]{\includegraphics[width=0.1798\linewidth]{\currentname}}\hfill
\input{Im_haarTIApprox_02_Saturn_legend}  \subfigure[\currentcaption]{\includegraphics[width=0.1798\linewidth]{\currentname}}\hfill
\input{Im_SAFIR_02_Saturn_legend}  \subfigure[\currentcaption]{\includegraphics[width=0.1798\linewidth]{\currentname}}\hfill
\input{Im_BM3D_02_Saturn_legend}  \subfigure[\currentcaption]{\includegraphics[width=0.1798\linewidth]{\currentname}}\hfill
\input{Im_AnscombePCA_02_Saturn_legend}  \subfigure[\currentcaption]{\includegraphics[width=0.1798\linewidth]{\currentname}}\hfill
\input{Im_NLPCA_02_Saturn_legend}  \subfigure[\currentcaption]{\includegraphics[width=0.1798\linewidth]{\currentname}}\hfill
\input{Im_NLPCAS_02_Saturn_legend}  \subfigure[\currentcaption]{\includegraphics[width=0.1798\linewidth]{\currentname}}\hfill
\input{Im_BM3Dbin_02_Saturn_legend}  \subfigure[\currentcaption]{\includegraphics[width=0.1798\linewidth]{\currentname}}\hfill
\input{Im_NLPCASbin_02_Saturn_legend}  \subfigure[\currentcaption]{\includegraphics[width=0.1798\linewidth]{\currentname}}\hfill
\caption{Toy cartoon image (Saturn) corrupted with Poisson noise with  Peak = 0.2. 
}
\label{fig:Saturn_peak=02}
\end{figure*}

\begin{table*}[htbp!]
\begin{center}
{\scriptsize
\begin{tabular}{lc@{\quad}c@{\quad}c@{\quad}c@{\quad}c@{\quad}c@{\quad}c@{\quad}c@{\quad}}
\toprule
Method 
& Swoosh
& Saturn
& Flag
& House
& Cam
& Man
& Bridge
& Ridges
\\
\midrule
& \multicolumn{8}{c}{ Peak $ = 0.1$}\\
\cmidrule{2-9} 
NLBayes 
& 11.08
& 12.65
& 7.14
& 10.94
& 10.54
& 11.52
& 10.58
& 15.97
\\
haarTIApprox 
& 19.84
& 19.36
& 12.72
& 18.15
& 17.18
& \textbf{19.10}
& 16.64
& 18.68
\\
SAFIR 
& 18.88
& 20.39
& 12.24
& 17.45
& 16.22
& 18.53
& 16.55
& 17.97
\\
BM3D 
& 17.21
& 19.13
& 13.12
& 16.63
& 15.75
& 17.24
& 15.72
& 19.47
\\
BM3Dbin 
& \textbf{21.91}
& \textbf{20.82}
& 14.36
& 18.39
& 17.11
& 18.84
& \textbf{16.94}
& 20.33
\\
NLPCA 
& 19.12
& 20.40
& 14.45
& 18.06
& 16.58
& 18.48
& 16.48
& 21.25
\\
NLSPCA 
& 19.18
& 20.45
& 14.50
& 18.08
& 16.64
& 18.49
& 16.52
& 20.56
\\
NLSPCAbin 
& 21.56
& 19.47
& \textbf{15.57}
& \textbf{18.68}
& \textbf{17.29}
& 18.73
& 16.90
& \textbf{23.52}
\\
\midrule
& \multicolumn{8}{c}{ Peak $ = 0.2$}\\
\cmidrule{2-9} 
NLBayes 
& 14.18
& 14.75
& 8.20
& 13.54
& 12.71
& 13.89
& 12.59
& 16.19
\\
haarTIApprox 
& 21.55
& 20.91
& 13.97
& 19.25
& \textbf{18.37}
& 20.13
& 17.46
& 20.46
\\
SAFIR 
& 20.86
& 21.71
& 13.65
& 18.83
& 17.38
& 19.88
& 17.41
& 18.58
\\
BM3D 
& 20.27
& 21.20
& 14.25
& 18.67
& 17.44
& 19.31
& 17.14
& 21.10
\\
BM3Dbin 
&\textbf{ 24.14}
& \textbf{22.59}
& 16.04
& \textbf{19.93}
& 18.24
& \textbf{20.22}
& \textbf{17.66}
& 23.92
\\
NLPCA 
& 21.20
& 22.29
& 16.53
& 19.08
& 17.80
& 19.69
& 17.49
& 24.10
\\
NLSPCA 
& 21.27
& 22.34
& 16.47
& 19.11
& 17.77
& 19.70
& 17.51
& 24.41
\\
NLSPCAbin 
& 24.04
& 20.56
&\textbf{16.65}
& 19.87
& 17.90
& 19.61
& 17.43
& \textbf{25.43}
\\
\midrule
& \multicolumn{8}{c}{ Peak $ = 0.5$}\\
\cmidrule{2-9} 
NLBayes 
& 19.60
& 18.28
& 10.19
& 17.01
& 15.68
& 16.90
& 15.11
& 16.77
\\
haarTIApprox 
& 23.59
& 23.27
& 16.25
& 20.65
& \textbf{19.59}
& 21.30
& 18.32
& 23.07
\\
SAFIR 
& 22.70
& 24.23
& 16.20
& 20.37
& 18.84
& 21.25
& 18.42
& 20.90
\\
BM3D 
& 23.53
& 24.09
& 15.94
& 20.50
& 18.86
& 21.03
& 18.37
& 23.33
\\
BM3Dbin 
& 26.20
& \textbf{25.64}
& 18.53
& \textbf{21.70}
& 19.58
& \textbf{21.60}
& \textbf{18.75}
& 27.99
\\
NLPCA 
& 24.50
& 25.38
& \textbf{18.93}
& 20.78
& 19.36
& 21.13
& 18.47
& \textbf{28.06}
\\
NLSPCA 
& 24.44
& 25.06
& 18.92
& 20.76
& 19.23
& 21.12
& 18.46
& 28.03
\\
NLSPCAbin 
& \textbf{26.36}
& 20.67
& 17.09
& 20.97
& 18.39
& 20.28
& 18.16
& 26.81
\\
\midrule
& \multicolumn{8}{c}{ Peak $ = 1$}\\
\cmidrule{2-9} 
NLBayes 
& 23.58
& 21.66
& 14.00
& 19.27
& 17.99
& 19.48
& 16.85
& 18.35
\\
haarTIApprox 
& 25.12
& 25.06
& 17.79
& 21.97
& \textbf{20.64}
& 22.25
& 19.08
& 24.52
\\
SAFIR 
& 23.37
& 25.14
& 17.91
& 21.46
& 20.01
& 22.08
& 19.12
& 24.67
\\
BM3D 
& 26.21
& 25.88
& 18.45
& 22.26
& 20.45
& 22.27
& 19.39
& 25.76
\\
BM3Dbin 
& \textbf{27.95}
& \textbf{27.24}
& 19.49
& \textbf{ 23.26}
& 20.61
& \textbf{22.53}
& \textbf{19.47}
& 29.91
\\
NLPCA 
& 26.99
& 27.08
& 20.23
& 22.07
& 20.31
& 21.96
& 19.01
& \textbf{30.17}
\\
NLSPCA 
& 27.02
& 27.04
& \textbf{20.37}
& 22.10
& 20.28
& 21.88
& 19.00
& 30.04
\\
NLSPCAbin 
& 27.21
& 21.10
& 17.03
& 21.21
& 18.45
& 20.37
& 18.36
& 26.96
\\
\midrule
& \multicolumn{8}{c}{ Peak $ = 2$}\\
\cmidrule{2-9} 
NLBayes 
& 27.50
& 24.66
& 17.13
& 21.10
& 19.67
& 21.34
& 18.22
& 21.04
\\	
haarTIApprox 
& 27.01
& 26.43
& 19.33
& 23.37
& 21.72
& 23.18
& 19.90
& 26.53
\\
SAFIR 
& 23.78
& 26.02
& 19.25
& 22.33
& 21.30
& 22.74
& 19.99
& 28.29
\\
BM3D 
& 28.63
& 27.70
& 20.66
& 24.25
& \textbf{22.19}
& \textbf{23.54}
& \textbf{20.44}
& 29.75
\\
BM3Dbin 
& \textbf{29.70}
& \textbf{28.68}
& 20.01
& \textbf{24.52}
& 21.42
& 23.43
& 20.17
& 32.24
\\
NLPCA 
& 29.41
& 28.02
& 20.64
& 23.44
& 20.75
& 22.78
& 19.37
& 32.25
\\
NLSPCA 
& 29.53
& 28.11
& \textbf{20.75}
& 23.75
& 20.76
& 22.86
& 19.45
& \textbf{32.35}
\\
NLSPCAbin 
& 27.62
& 21.13
& 17.02
& 21.42
& 18.33
& 20.34
& 18.34
& 29.31
\\
\midrule
& \multicolumn{8}{c}{ Peak $ = 4$}\\
\cmidrule{2-9} 
NLBayes 
& 31.17
& 26.73
& 22.64
& 23.61
& 22.32
& 23.02
& 19.60
& 24.04
\\
haarTIApprox 
& 28.55
& 28.13
& 21.16
& 24.88
& 22.93
& 24.23
& 20.83
& 28.56
\\
SAFIR 
& 25.40
& 27.40
& 20.71
& 23.76
& 22.73
& 23.85
& 20.88
& 30.52
\\
BM3D 
& 30.36
& 29.30
& \textbf{22.91}
& \textbf{26.08}
& \textbf{23.93}
& \textbf{24.79}
& \textbf{21.50}
& 32.50
\\
BM3Dbin 
& 31.15
&\textbf{ 30.07}
& 20.57
& 25.64
& 22.00
& 24.28
& 20.84
& 33.52
\\
NLPCA 
& 31.08
& 29.07
& 20.96
& 24.49
& 20.96
& 23.18
& 19.73
& \textbf{33.73}
\\
NLSPCA 
&\textbf{ 31.46}
& 29.51
& 21.15
& 24.89
& 21.08
& 23.41
& 20.15
& 33.69
\\
NLSPCAbin 
& 27.65
& 21.45
& 16.00
& 21.47
& 18.44
& 20.35
& 18.35
& 29.13
\\

\bottomrule
\end{tabular}
}
\caption{Experiments on simulated data (average over five noise realizations). Flag and Saturn images are displayed 
in Figs.~\ref{fig:Saturn_peak=02}, \ref{fig:Flag_peak=01} and \ref{fig:Flag_peak=1}, and the others are given in 
\cite{Salmon_Willett_AriasCastro12} and in
\cite{Zhang_Fadili_Starck08}.
\label{tab:optimal_bandwidth} 
} 
\end{center}
\end{table*}

\begin{figure*}
\centering
\subfigure[Original, channel 68 ]{\includegraphics[width=0.19\linewidth]{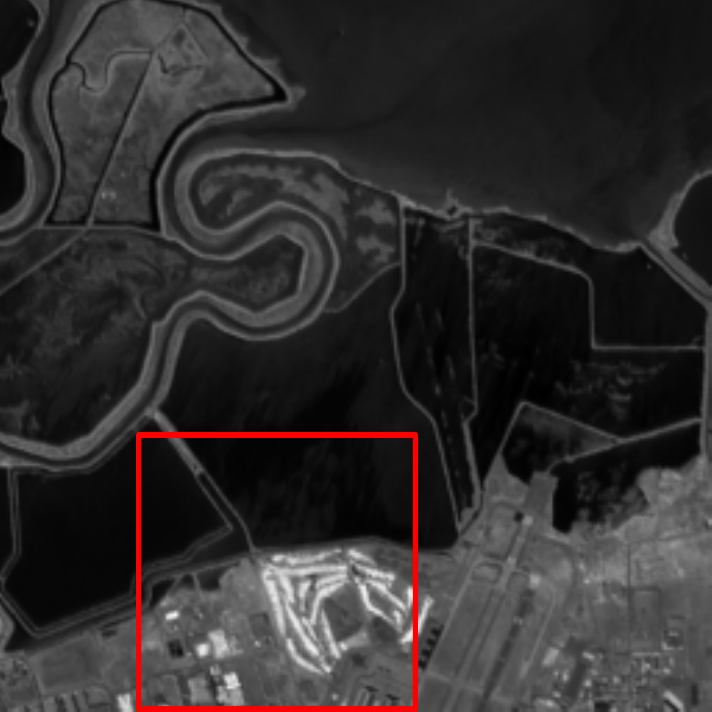}}
\hfill
\subfigure[Noisy data]{\includegraphics[width=0.19\linewidth]{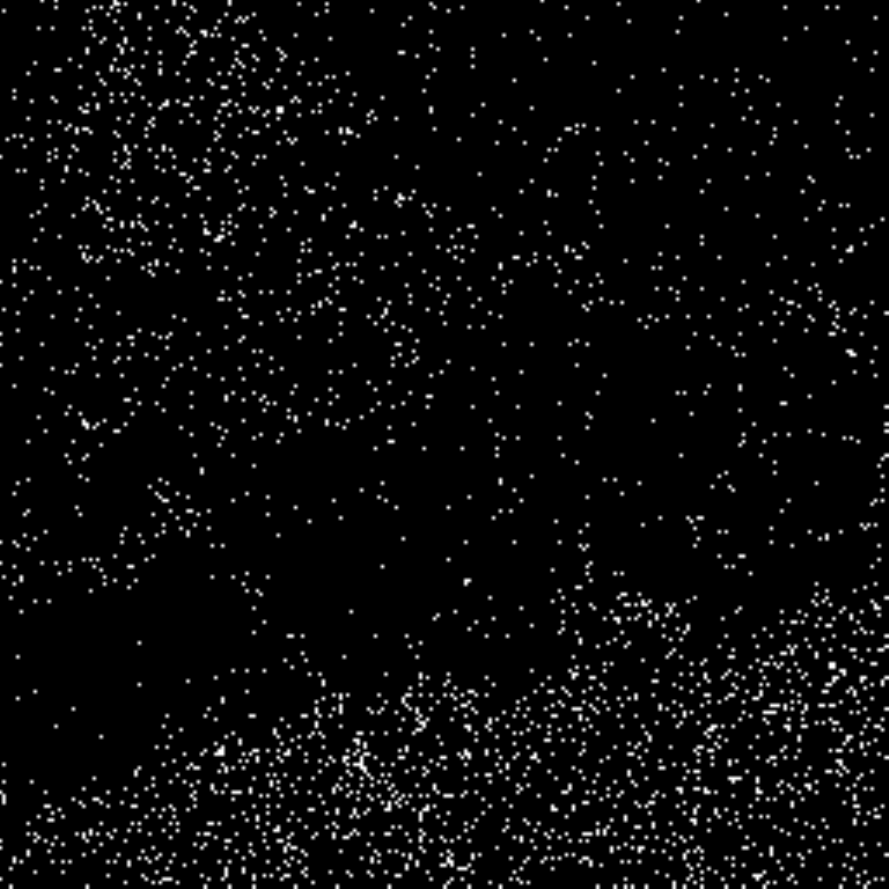}} 
\hfill
\subfigure[BM4D, $4 \times 4 \times 16$ MAE=0.2426]{\includegraphics[width=0.19\linewidth]{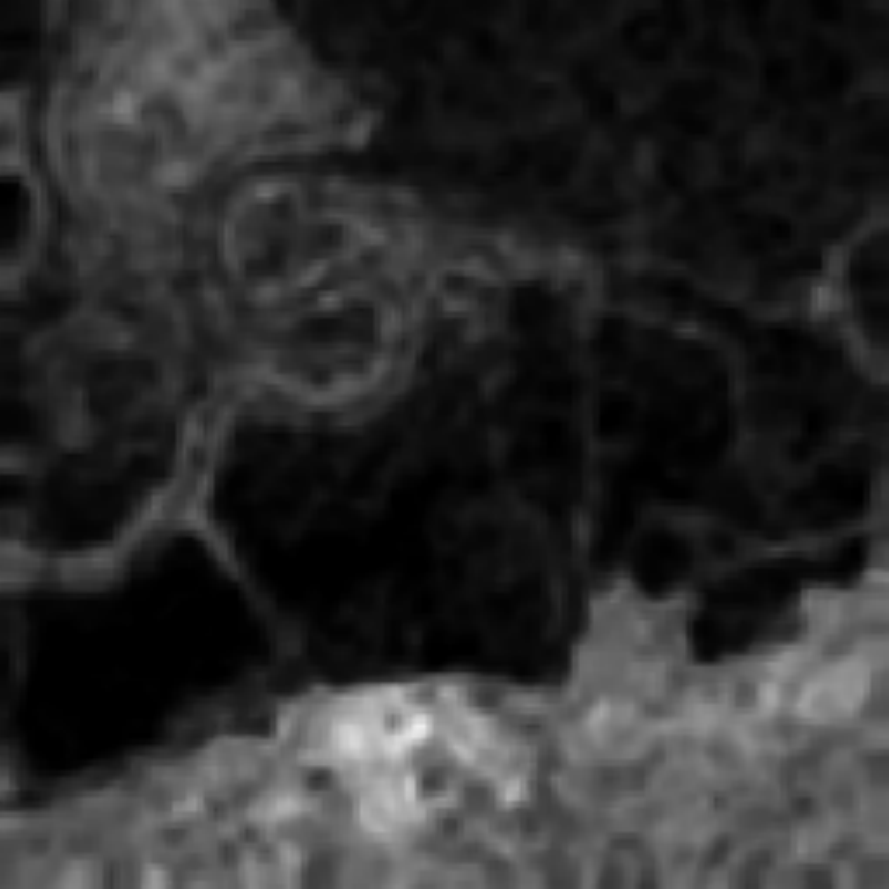}}
\hfill
\subfigure[Multiscale partition, MAE=0.1937]{\includegraphics[width=0.19\linewidth]{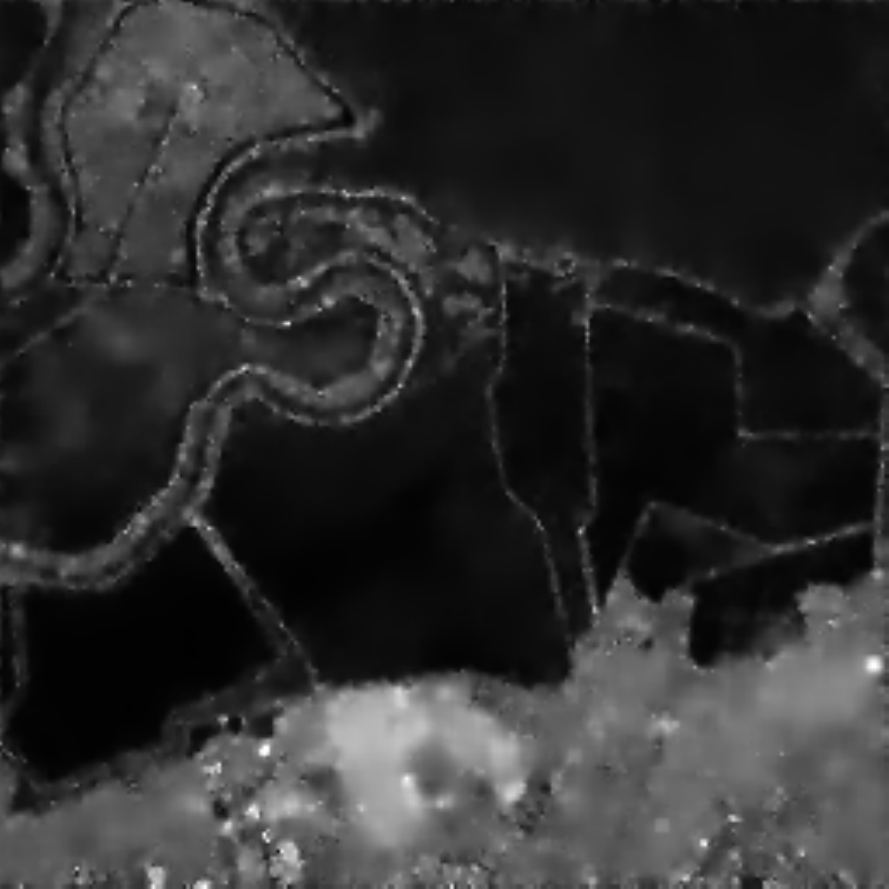}} 
\hfill
\subfigure[NLSPCA, $5 \times 5 \times 23$, MAE=0.1893]{\includegraphics[width=0.19\linewidth]{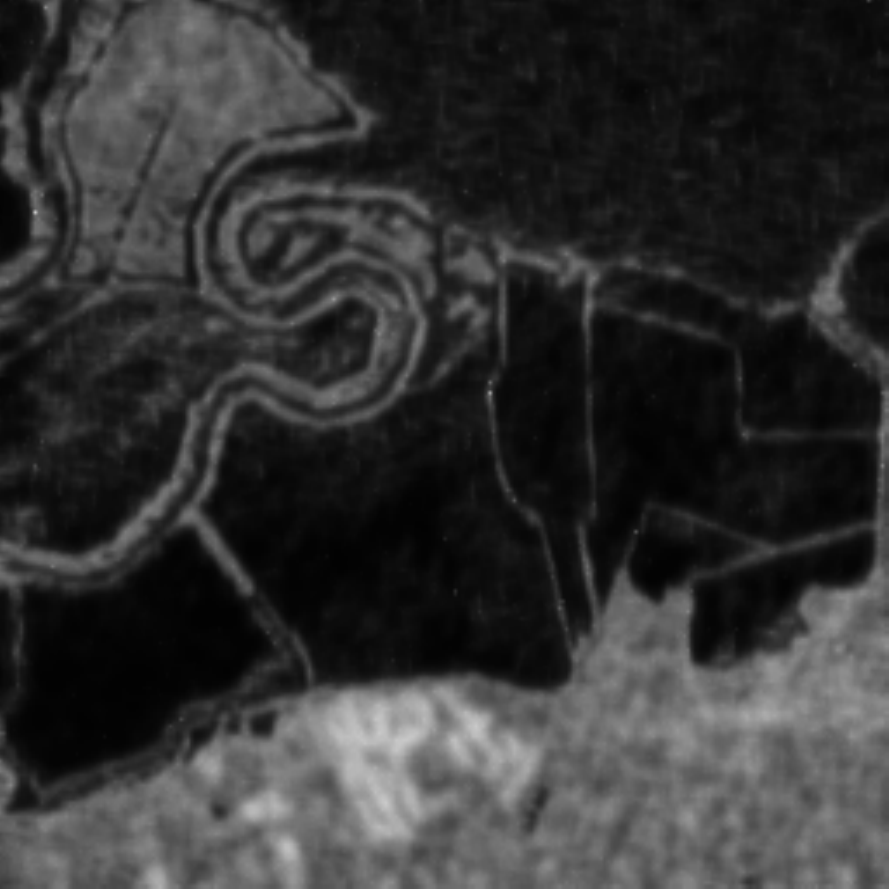}} 
%
%
\\
\subfigure[Original, channel 68]{\includegraphics[width=0.19\linewidth,clip,viewport=50 0 150 100]{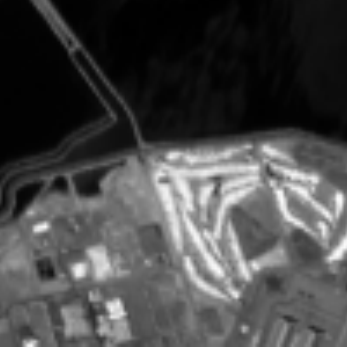}}
\hfill
\subfigure[Noisy data ]{\includegraphics[width=0.19\linewidth,clip,viewport=50 0 150 100]{noisy_compMoffettFieldpeak03396Patch_width7_68}} 
\hfill
%
%
\subfigure[BM4D, $4 \times 4 \times 16$, MAE=0.2426]{\includegraphics[width=0.19\linewidth,clip,viewport=50 0 150 100]{bm4d_compMoffettFieldkal_bm4d_68}}
\hfill
\subfigure[Multiscale partition, MAE=0.1937,]{\includegraphics[width=0.19\linewidth,clip,viewport=50 0 150 100]{fil_kal_compMoffettFieldkal_bm4d_68}} 
\hfill
\subfigure[NLSPCA, $5 \times 5 \times 23$, MAE=0.1893]{\includegraphics[width=0.19\linewidth,clip,viewport=50 0 150 100]{NLPCA_compMoffettFieldpeak03396Patch_width5Patch_width_3d23__68}} 
\caption{Original and close-up of the red square from spectral band 68 of the Moffett Field. 
The same methods are considered, and are displayed
 in the same order: original, noisy (with  0.0387 photons per voxels), 
BM4D \cite{Maggioni_Katkovnik_Egiazarian_Foi11} (with inverse Anscombe as in \cite{Makitalo_Foi11}),
multiscale partitioning method \cite{Krishnamurthy_Raginsky_Willett10}, and our proposed method with 
patches of size $5\times5\times 23$.
}
\label{fig:moffet_close_up}
\end{figure*}

\begin{figure*}
\centering
\subfigure[Noisy (channel 101) ]{\includegraphics[width=0.19\linewidth]{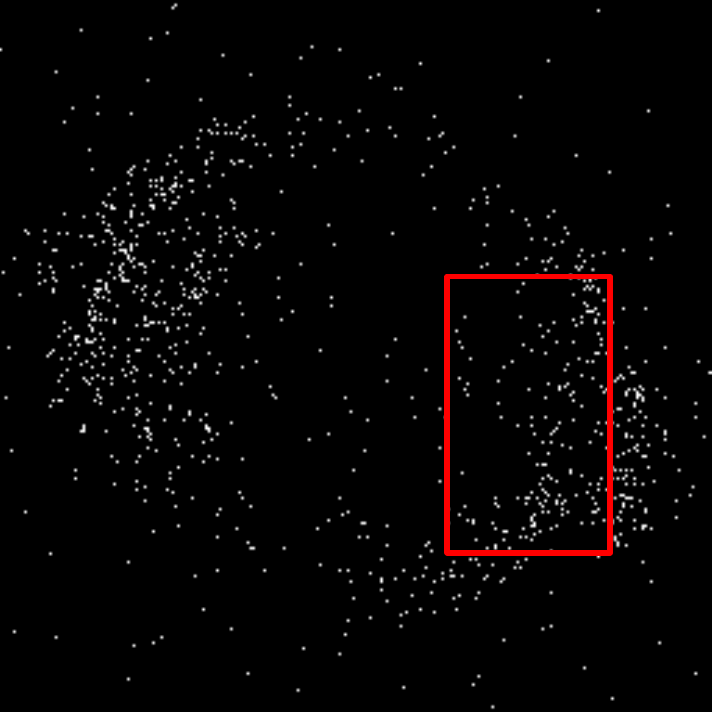}} 
\hfill
\subfigure[Average over channels]{\includegraphics[width=0.19\linewidth]{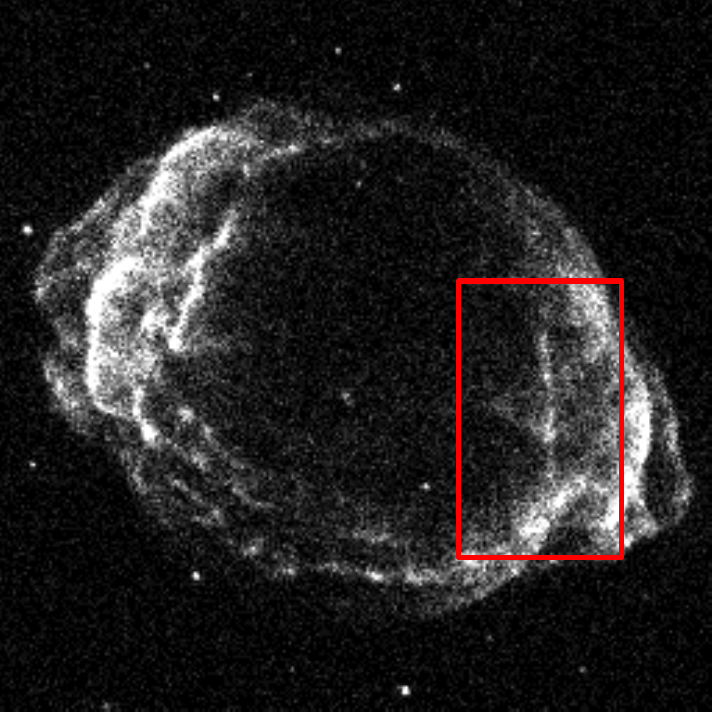}} 
\hfill
\subfigure[BM4D, $4 \times 4 \times 16$   ]{\includegraphics[width=0.19\linewidth]{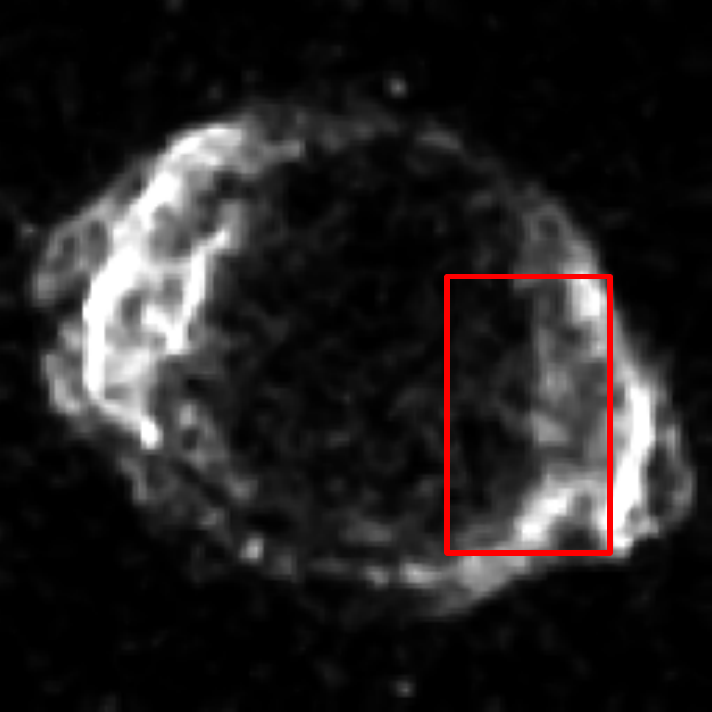}}
\hfill
\subfigure[Multiscale partition]{\includegraphics[width=0.19\linewidth]{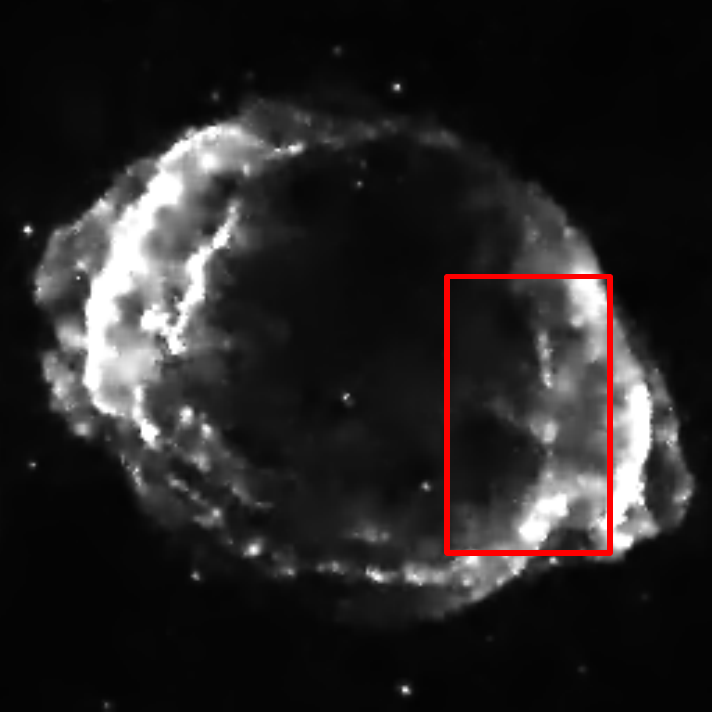}} 
\hfill
\subfigure[NLSPCA, $5 \times 5 \times 23$  ]{\includegraphics[width=0.19\linewidth]{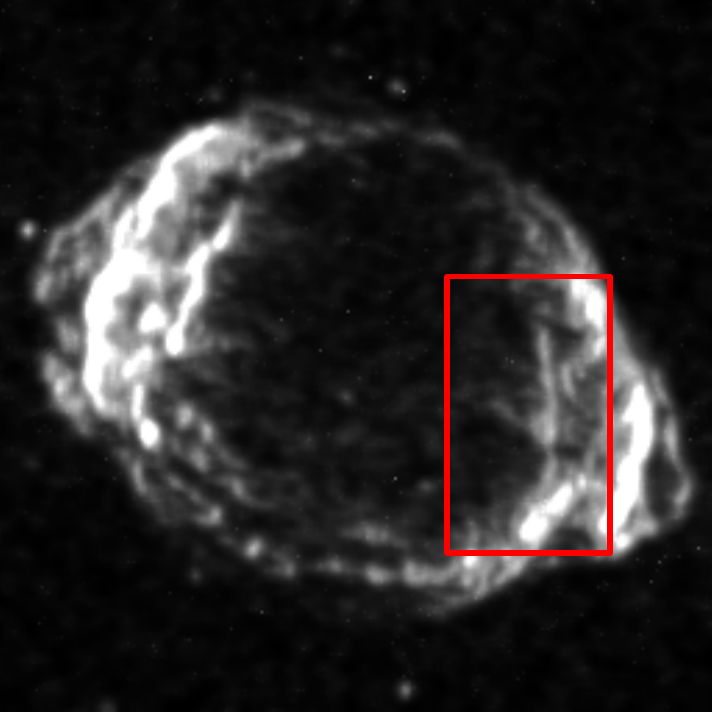}}

\caption{Spectral image of the supernova remnant G1.9+0.3. We display
  the spectral band 101 of the noisy observation (with 0.0137 photons
  per voxels), and this denoised channel with BM4D
  \cite{Maggioni_Katkovnik_Egiazarian_Foi11} (with inverse Anscombe as
  in \cite{Makitalo_Foi11}), the multiscale partitioning method
  \cite{Krishnamurthy_Raginsky_Willett10}, and our proposed method
  NLSPCA with patches of size $5\times5\times 23$.  Note how the highlighted detail shows structure in the
average over channels, which appears to be accurately reconstructed by our method.}
\label{fig:chandra}
\end{figure*}


\section{Conclusion and future work}\label{seq:conclusion}
Inspired by the methodology of \cite{Deledalle_Salmon_Dalalyan11} we
have adapted a generalization of the PCA
\cite{Collins_Dasgupta_Schapire02,Roy_Gordon_Thrun05} for denoising
images damaged by Poisson noise.  In general, our method finds a
good rank-$\ell$ approximation to each cluster of patches. While this can be done
either in the original pixel space or in a logarithmic ``natural parameter'' space, we
choose the logarithmic scale to avoid issues with nonnegativity, facilitating fast algorithms.  One might ask whether working on a
logarithmic scale impacts the accuracy of this rank-$\ell$
approximation.  Comparing against several state-of-the-art approaches,
we see that because our approach often works as well or better than these
alternatives, the exponential formulation of PCA does not lose significant
approximation power or else it would manifest itself in these results.

Possible improvements include adapting the number of dictionary
elements used with respect to the noise level, and proving a
theoretical convergence guarantees for the algorithm. The nonconvexity
of the objective may only allow convergence to local minima.
 An open question is whether these local minima have interesting properties. 
 Reducing the computational
complexity of NLPCA is a final remaining challenge.


\section*{Acknowledgments} {Joseph Salmon, Zachary Harmany, and
  Rebecca Willett gratefully acknowledge support from DARPA grant no.\
  FA8650-11-1-7150, AFOSR award no.\ FA9550-10-1-0390, and NSF award
  no.\ CCF-06-43947.  The authors would also like to thank
  J. Boulanger and C. Kervrann for providing their SAFIR algorithm,
  Steven Reynolds for providing the spectral images from the
  supernova remnant G1.9+0.3, and an anonymous reviewer for proposing 
  the improvement using the binning step.}

\section*{Appendix}

\appendix


\section{Biconvexity of loss function}
\begin{lemma}
 The function $L$ is biconvex with respect to $(U,V)$ but not jointly convex.
\end{lemma}
\begin{proof} The biconvexity argument is straightforward; the partial functions 
$U \mapsto L(U,V)$ with a fixed $V$ and 
$V \mapsto L(U,V)$ with a fixed $U$  are both convex. The fact that the problem is non-jointly convex can be seen when $U$ and $V$ are in $\R$ (\ie $\ell=m=n=1$), since
the Hessian in this case is
\begin{equation*}
H_L(U,V) =
 \begin{pmatrix}
  V^2 e^{UV}  & UV e^{UV}+e^{UV}-Y\\
 UV e^{UV}+e^{UV}-Y  & U^2 e^{UV} 
   \end{pmatrix}.
\end{equation*}
Thus at the origin one has $H_L(0,0)= \begin{pmatrix}
  0 & 1\\
  1 & 0
   \end{pmatrix}$, which has  a negative eigenvalue, $-1$.
\end{proof}


\section{Gradient calculations}

We provide below the gradient computation used in Eq.~\eqref{eq:update_row_U_poisson} and Eq.~\eqref{eq:update_row_V_poisson}:
\begin{align*}
\nabla_U L(U,V)&=(\exp(UV) -Y)V^\top \,, \\ 
\nabla_V L(U,V)&=U^\top (\exp(UV) -Y) \,.
\end{align*}


Using the component-wise representation this is equivalent to 
\begin{align*}
\frac{\partial L(U,V)}{\partial U_{a,b}}&=\displaystyle\sum_{j=1}^N  \exp(UV)_{a,j} V_{b,j} - Y_{a,j} V_{b,j} \,, \\
\frac{\partial L(U,V)}{\partial V_{a,b}}&=\displaystyle\sum_{i=1}^M U_{i,a} \exp(UV)_{i,b} - U_{i,a}Y_{i,b} \,.
\end{align*}

\section{Hessian calculations} \label{app:hess}

The approach proposed by 
\cite{Gordon03,Roy_Gordon_Thrun05} 
consists in using an iterative algorithm which sequentially updates the $j$th column of $V$ and the $i$th row of $U$.
The only problem with this method is numerical: one needs to invert  possibly ill conditioned matrices at each step 
of the loop.

The Hessian matrices of our problems, with respect to $U $ and $V$ respectively are given by
\begin{equation*}
\frac{\partial^2 L(U,V)}{\partial U_{a,b} \partial U_{c,d}}=\left\{
    \begin{array}{ll}
        \displaystyle\sum_{j=1}^N  \exp(UV)_{a,j} V^2_{b,j}, & \mbox{if } (a,b)=(c,d), \\
        0 & \mbox{otherwise,}
    \end{array}
\right. 
\end{equation*}
and
\begin{equation*}
\frac{\partial^2 L(U,V)}{\partial V_{a,b} \partial V_{c,d}}=\left\{
    \begin{array}{ll}
        \displaystyle\sum_{i=1}^M  U^2_{i,a} \exp(UV)_{i,b}, &\mbox{if } (a,b)=(c,d), \\
        0&\mbox{otherwise.}
    \end{array}
\right. 
\end{equation*}
Notice that both Hessian matrices are diagonal. So applying the
inverse of the Hessian simply consists in inverting the diagonal
coefficients.

\section{The Newton step}\label{app:newton}

In the following we need to introduce the function 
$ \Vect_C$ that transforms a matrix into one single column (concatenates the columns), and the function $ \Vect_R$ that transforms
a matrix into a single row (concatenates the rows). This means that
\begin{align*}
\Vect_C: \, & \R^{M \times\ell} &\longrightarrow   & ~\R^{M \ell \times 1} \,, \nonumber\\
 &U=(U_{1,:}, \cdots,U_{\ell,:}) &\longmapsto  & ~(U_{1,:}^\top, \cdots,U_{\ell,:}^\top)^\top,
\end{align*}and 
\begin{align*}
\Vect_R:\, & \R^{\ell\times N} &\longrightarrow & ~\R^{ 1\times \ell  N} \,, \nonumber\\
 &V=(V_{:,1}^\top, \cdots,V_{:,\ell}^\top)^\top &\longmapsto & ~(V_{:,1}, \cdots,V_{:,\ell}) .
\end{align*}

Now using the previously introduced notations, the updating steps for $U$ and $V$ can  be written
\begin{align}
\Vect_C(U_{t+1})&=\Vect_C(U_t) -H_{U_t}^{-1} \Vect_C \big(\nabla_{U}L(U_t,V_t)\big) \,, \label{eq:vectcu}\\ 
\Vect_R(V_{t+1})&=\Vect_R(V_t) - \Vect_R \big(\nabla_{V}L(U_t,V_t)\big)H_{V_t}^{-1} \,.
\end{align}
We give the order used to concatenate the coefficients for the Hessian matrix  with respect to $U$, $H_U$: 
 $(a,b)=(1,1), \cdots,(M,1),(1,2),\cdots(M,2), \cdots (1,\ell), \cdots, (M,\ell)$.
 
\noindent We concatenate the column of $U$ in this order.

It is easy to give the updating rules for the $k$th column of $U$, 
one only needs to multiply the first Equation of \eqref{eq:vectcu} from the left by the $M \times M\ell$ matrix 
\begin{equation}
F_{k,M,\ell,} =
 \begin{pmatrix}
   0_{M,M}, & \cdots,& I_{M,M},&\cdots,& 0_{M,M} 
   \end{pmatrix}
\end{equation}
where the identity block matrix is in the $k$th position.
This leads to the following updating rule
\begin{equation}
{U_{t+1,\cdot,k}=U_{t,:,k} -D_k^{-1} (\exp(U_tV_t) -Y)V_{t,k,:}^\top } ~,
\end{equation} 
where $D_k$ is a diagonal matrix of size  $M\times M$:
\begin{equation*}\begin{split}D_k=\diag \Big(\displaystyle \sum_{j=1}^n    \exp(U_t V_t)_{1,j} &V^2_{t,k,j},\cdots,\\
&\displaystyle\sum_{j=1}^n\exp(U_tV_t)_{M,j} V^2_{t,k,j} \Big).\end{split}
\end{equation*}
This leads easily to \eqref{eq:update_row_U_poisson}.


By the symmetry of the problem in $U$ and $V$, one has the following equivalent updating rule for $V$:
\begin{equation}
{V_{t+1,k,:}=V_{t,k,:} -U_{t,:,k}^\top (\exp(U_tV_t) -Y) E_{k,M}^{-1}  } ~,
\end{equation}
where $E_k$ is a diagonal matrix of size $ N \times N$:
\begin{equation*}\begin{split}E_{k}=\diag \Big(\displaystyle\sum_{i=1}^M  \exp(U_t V_t)_{i,1} &U^2_{t,i,k},\cdots,\\
&\displaystyle\sum_{j=1}^n \exp(U_tV_t)_{i,n}   U^2_{t,i,k}  \Big) .\end{split}
\end{equation*}

\bibliographystyle{plain}
\bibliography{references_all}

\end{document}